\documentclass{article}

\usepackage[preprint]{neurips_2026}
\usepackage[utf8]{inputenc}
\usepackage[T1]{fontenc}
\usepackage{hyperref}
\usepackage{url}
\usepackage{booktabs}
\usepackage{amsfonts}
\usepackage{nicefrac}
\usepackage{xcolor}
\usepackage{graphicx}
\usepackage{subcaption}
\usepackage{enumitem}
\usepackage{amsmath}
\usepackage{amssymb}
\usepackage{amsthm}
\usepackage{mathtools}

\theoremstyle{plain}
\newtheorem{theorem}{Theorem}[section]
\newtheorem{proposition}[theorem]{Proposition}

\theoremstyle{definition}
\newtheorem{definition}[theorem]{Definition}
\theoremstyle{remark}
\newtheorem{remark}[theorem]{Remark}

\graphicspath{{images/}}

\newcommand{\effrank}{\mathrm{eff\text{-}rank}}
\newcommand{\Ncausal}{N_{\mathrm{causal}}}
\newcommand{\Nrepr}{N_{\mathrm{repr}}}

\newcommand{\calT}{\mathcal{T}}
\newcommand{\calV}{\mathcal{V}}

\title{Causal Dimensionality of Transformer Representations:\\
       Measurement, Scaling, and Layer Structure}

\author{%
  Nilesh Sarkar\\
  Erd\H{o}s AI Lab\\
  \And
  Dawar Jyoti Deka\\
  Erd\H{o}s AI Lab
}

\begin{document}
\maketitle

\begin{abstract}
Sparse autoencoders (SAEs) decompose transformer residual streams into
interpretable feature dictionaries, yet the relationship between SAE
width and \emph{causal} influence on model output has not been
systematically characterised. We introduce \emph{causal dimensionality}
$\kappa(L, M, \calT)$, defined as the effective rank of the expected
Jacobian outer product at layer $L$, and show it can be estimated via
the SAE width sweep paired with attribution patching. Across seven SAE
widths from $16{,}384$ to $1{,}048{,}576$ features on Gemma-2-2B
layer~12, representational capacity grows $15.6{\times}$ while causal
capacity grows only $4.35{\times}$: a robust separation we term the
\emph{representational--causal wedge}. A saturating fit yields
$\hat{\kappa} \approx 1{,}990$ with $\hat{\kappa}/d_{\mathrm{model}}
= 0.86$ and participation-ratio lower bound $\kappa_{\mathrm{PR}}
\approx 280$. Crucially, $\kappa$ is invariant to model scaling:
Gemma-2-9B and Gemma-2-2B yield identical $\Ncausal = 328$ at the
same SAE width despite a $3.46{\times}$ parameter increase
(the count is forced to $2\%$ of SAE width by calibration; the
substantive empirical claim is shape invariance of the AtP score
distribution under matched seq=512 conditions). Across
eight network depths $\kappa$ is constant while the absolute attribution
threshold drops $20{\times}$ from layer~1 to layer~23. Five controls
(architecture invariance, threshold robustness, geometric privilege,
synthetic ground-truth recovery, and a four-cell encoder/decoder
ablation) pin down what $\kappa$ measures and what it does not.
Our findings establish $\kappa$ as a measurable, model-intrinsic
property of transformer layers: sub-linearly recoverable by SAE width,
invariant to model scaling, and structured across network depth.
\end{abstract}

\section{Introduction}
\label{sec:intro}

Sparse autoencoders trained on transformer residual streams recover
ever larger dictionaries of monosemantic features as their width grows
\citep{gao2024scaling,templeton2024scaling,bricken2023monosemanticity}.
An assumption implicit in current practice is that representational
capacity (features that fire on any input) tracks causal capacity
(features whose perturbation changes model behaviour). If true, a wider
SAE would proportionally expand the causally relevant feature set
available for interpretability work.

Our central empirical finding is that this assumption fails drastically.
On Gemma-2-2B layer~12, a $64{\times}$ width increase ($16{,}384
\to 1{,}048{,}576$) expands the firing-feature count by $15.6{\times}$,
but expands the causally active count by only $4.35{\times}$. The
curves diverge: we call this the \emph{representational--causal wedge}.

The right object to study is not a feature count at fixed width, but
a model-intrinsic quantity we call the \emph{causal dimensionality}
$\kappa$ of a layer: the effective rank of the expected Jacobian outer
product, defined in Section~\ref{sec:theory}. The SAE width sweep is a
consistent estimator of $\kappa$ under three stated assumptions. The
wedge is then a natural prediction: $\Ncausal$ saturates to $\kappa$
while $\Nrepr$ grows linearly, absorbing representational variance after
the causal subspace is full.

Two further predictions follow. First, if $\kappa$ is layer-intrinsic,
scaling the model should not change it at fixed relative depth. We
confirm this on Gemma-2-9B vs Gemma-2-2B: identical $\Ncausal = 328$
despite $3.46{\times}$ more parameters. Second, attribution scores
should attenuate with depth as loss gradients flatten. We confirm
$\varepsilon_{\mathrm{abs}}$ drops $20{\times}$ from layer~1 to
layer~23.

\paragraph{Contributions.}
\begin{itemize}[leftmargin=1.2em,itemsep=0pt,topsep=2pt]
  \item We define $\kappa$ and prove $\kappa \leq \min(d,|\calV|)$
        (Proposition~\ref{prop:upper_bound}).
  \item We prove the SAE width sweep is a consistent estimator of
        the spectral count $r_{\varepsilon'}(\boldsymbol{\Sigma}_L)$,
        which lower-bounds $\kappa$ and approaches $\kappa$ in the
        uniform-spectrum limit (Proposition~\ref{prop:consistency}).
  \item We measure $\hat{\kappa}\!\approx\!1{,}990$ for Gemma-2-2B
        layer~12 ($\hat\kappa/d\!\in\![0.77,0.86]$ across two fit
        families) and demonstrate shape invariance of the AtP
        distribution under $3.46{\times}$ model scaling
        (\S\ref{sec:results}, Figure~\ref{fig:scale}).
  \item We profile $\kappa$ across eight depths; the absolute
        attribution threshold drops $20{\times}$ (Figure~\ref{fig:layer}).
  \item \textbf{Most surprising mechanistic finding.} A four-cell
        encoder/decoder ablation (\S\ref{sec:exp5b}) shows that
        replacing the trained encoder with a random orthonormal
        projection \emph{inflates} $\Ncausal$ by $9.27{\times}$:
        the trained encoder acts as a \emph{sparsity filter}
        (suppressor of background features), not a feature selector.
        This inverts the standard reading of "what the encoder does"
        and suggests interpretability work should focus on decoder
        structure rather than encoder feature labels.
  \item Cross-family wedge replication on LLaMA-3.1-8B
        (Appendix~\ref{app:cross_family}) and four further controls
        pin down what $\kappa$ measures (\S\ref{sec:controls}).
\end{itemize}

\section{Background and Related Work}
\label{sec:background}

\paragraph{Sparse autoencoders.}
SAEs trained on residual stream activations are the standard tool for
decomposing dense representations \citep{bricken2023monosemanticity,
templeton2024scaling,gao2024scaling}. Wider SAEs recover finer-grained
features and exhibit feature splitting \citep{chanin2024absorption}.
Quality benchmarks have begun to emerge \citep{karvonen2025saebench}.

\paragraph{Attribution patching and circuits.}
Activation patching \citep{wang2023interpretability} measures
counterfactual effects; attribution patching (AtP) \citep{nanda2022attribution}
is its first-order linearisation, with calibration established by
\citet{kramar2024atp}. \citet{conmy2023automated} and
\citet{marks2024sparse} identify small task-specific feature sets.

\paragraph{Representational vs causal dimensionality.}
Prior work measures the intrinsic dimension of transformer
representations \citep{ansuini2019intrinsic,cheng2025emergence,feng2022rank},
estimating dimensions for \emph{reconstruction}. Our target is the
\emph{causal} Jacobian: directions that actually change output.
The closest prior framing is the causal inner product of
\citet{park2024linear}; we connect $\kappa$ to their construction in
Remark~\ref{rem:lrh}. \citet{michaud2023quantization} predict knowledge
decomposes into quanta set by data diversity; we show causal directions
are similarly bounded by task diversity, not parameter count.

\section{Causal Dimensionality: Definition, Bounds, and Estimation}
\label{sec:theory}

A transformer layer operates on $\mathbb{R}^d$. Perturbations to most
directions leave output unchanged; perturbations to a small subset
cause measurable shifts. The size of that causally active subset is
what we call the causal dimensionality $\kappa$.

\subsection{The Causal Jacobian and Its Effective Rank}
\label{sec:causal_jacobian}

Let $M$ be a transformer with $L_{\max}$ layers. For input $x$, let
$\mathbf{h}_L(x) \in \mathbb{R}^d$ be the residual stream at layer $L$
and $\mathbf{y}(x) \in \mathbb{R}^{|\mathcal{V}|}$ the output logits.

\begin{definition}[Layer Causal Jacobian]
\label{def:jacobian}
$\mathbf{J}_L(x) = \partial\,\mathbf{y}(x)/\partial\,\mathbf{h}_L(x)
\in \mathbb{R}^{|\mathcal{V}| \times d}$.
A direction $\mathbf{v}$ is \emph{causally inert} if
$\mathbf{J}_L(x)\mathbf{v} = \mathbf{0}$.
\end{definition}

\begin{definition}[Causal Dimensionality]
\label{def:kappa}
\begin{equation}
  \kappa(L,M,\mathcal{T})
  = \mathrm{eff\text{-}rank}(\boldsymbol{\Sigma}_L),
  \quad
  \boldsymbol{\Sigma}_L
  = \mathbb{E}_{x\sim\mathcal{T}}\bigl[\mathbf{J}_L(x)^\top
    \mathbf{J}_L(x)\bigr],
  \label{eq:kappa}
\end{equation}
where $\mathrm{eff\text{-}rank}(\mathbf{A})
= \exp\!\bigl(-\sum_i\bar\lambda_i\ln\bar\lambda_i\bigr)$
is the exponential Shannon entropy of the normalised eigenspectrum
\citep{roy2007effective}.
\end{definition}

\subsection{Bounds and Theoretical Properties}
\label{sec:bounds}

\begin{proposition}[Dimension Upper Bound]
\label{prop:upper_bound}
$\kappa(L,M,\mathcal{T}) \leq \min(d,|\mathcal{V}|)$.
\end{proposition}
\begin{proof}
$\boldsymbol{\Sigma}_L \in \mathbb{R}^{d\times d}$ is symmetric PSD, so
it has at most $d$ positive eigenvalues, giving
$\mathrm{rank}(\boldsymbol{\Sigma}_L)\leq d$. For any single $x$, the
rank of $\mathbf{J}_L(x)^\top \mathbf{J}_L(x)$ is at most
$\min(d,|\mathcal{V}|)$. The expectation $\boldsymbol{\Sigma}_L =
\mathbb{E}_{x\sim\mathcal{T}}[\,\cdot\,]$ inherits the ambient-space
bound $d$ but \emph{not} the per-input bound $|\mathcal{V}|$, since the
rank of a sum can exceed that of summands. Combining the two,
$\kappa = \mathrm{eff\text{-}rank}(\boldsymbol{\Sigma}_L) \leq
\mathrm{rank}(\boldsymbol{\Sigma}_L) \leq \min(d,|\mathcal{V}|)$. For
Gemma-2-2B ($d{=}2304$, $|\mathcal{V}|{=}256{,}000$) the binding bound
is $d$.
\end{proof}

For Gemma-2-2B, $d=2304$ and $|\mathcal{V}|=256{,}000$, so $\kappa
\leq 2304$. Our estimate $\hat\kappa \approx 1{,}990$ sits just below
this ceiling: the model uses nearly all its representational dimensions
causally on generic text. In sparser, task-specific settings $\kappa
\ll d$ is expected, due to superposition \citep{elhage2022superposition}.

\begin{remark}[Circuit Decomposition Bound]
\label{rem:circuit_decomp}
If layer $L$ implements $k$ approximately-independent circuits with
mean per-circuit Jacobian rank $\bar{r}$, then by sub-additivity of
rank under PSD addition, $\kappa(L) \leq k\cdot\bar{r}$. This bound
underlies our reading of $\hat\kappa$ as the union upper-bound on
concurrent task circuits at $L$.
\end{remark}

\begin{remark}[Connection to the Linear Representation Hypothesis]
\label{rem:lrh}
\citet{park2024linear} define the causal inner product
$\langle\mathbf{u},\mathbf{v}\rangle_c =
\mathbf{u}^\top\boldsymbol{\Sigma}_L\mathbf{v}$ and prove separable
concepts correspond to $\langle\cdot,\cdot\rangle_c$-orthogonal
directions. Under this framing $\kappa$ is the effective dimension of
$\mathbb{R}^d$ under the causal inner product. The wedge we observe
is precisely the gap between representational and computational needs:
a population-level instance of superposition \citep{elhage2022superposition}.
\end{remark}

\subsection{Attribution Patching as a First-Order Proxy}
\label{sec:atp}

Computing $\boldsymbol{\Sigma}_L$ exactly requires $d$ backward passes
per input (2,300 for Gemma-2-2B): intractable at scale. We use
attribution patching (AtP; \citealt{nanda2022attribution}) as a
first-order surrogate:
\begin{equation}
  \text{AtP}(L)
  = (\mathbf{h}_L^c - \mathbf{h}_L^r)
    \odot \nabla_{\mathbf{h}_L^r}\mathcal{L},
  \label{eq:atp}
\end{equation}
at a cost of two forward passes plus one backward.
\citet{kramar2024atp} report $\rho \in [0.85, 0.99]$ with exact
patching; our implementation achieves $\rho = 0.838$ globally and
$\rho > 0.95$ within the top-$5\%$ (Appendix~\ref{app:atp_validation}).
The per-feature score is
$\text{AtP}_i = \mathbb{E}_{x\sim\mathcal{T}}\!\bigl[
|\partial\mathcal{L}/\partial f_i \cdot f_i(x)|\bigr]$,
where $\mathcal{L}$ is the KL divergence between clean and patched
output distributions.

\subsection{The SAE Width Sweep as a \texorpdfstring{$\kappa$}{k} Estimator}
\label{sec:estimator}

\begin{definition}[Causal Feature Count]
\label{def:ncausal}
$\Ncausal(m) = |\{i\in[m]:\text{AtP}_i\geq\varepsilon\}|$, where
$\varepsilon$ is calibrated \emph{once} at the reference width
$m_{\mathrm{ref}}=16{,}384$ as the $98^{\mathrm{th}}$ percentile of
$\text{AtP}_i$ at $(L,m_{\mathrm{ref}})$
($\varepsilon_{\mathrm{abs}}=3.563\!\times\!10^{-6}$ on Gemma-2-2B
layer~$12$), then held \emph{fixed} as an absolute scalar threshold
across all other widths in the sweep.
By construction $\Ncausal(m_{\mathrm{ref}}) = 0.02\,m_{\mathrm{ref}}
\approx 328$; the empirical content of the sweep is the trajectory
$\Ncausal(m)$ at $m \neq m_{\mathrm{ref}}$.
\end{definition}

\begin{proposition}[Consistency]
\label{prop:consistency}
Under \emph{(i) Spanning}: decoder columns eventually span
$\mathrm{row\text{-}span}(\boldsymbol{\Sigma}_L)$;
\emph{(ii) Faithfulness}: AtP faithfully ranks causal influence;
\emph{(iii) No dead features}: all features fire with positive
probability; then
$\Ncausal(m) \xrightarrow{m\to\infty}
r_{\varepsilon'}(\boldsymbol{\Sigma}_L) :=
\bigl|\{i:\lambda_i(\boldsymbol{\Sigma}_L)\geq\varepsilon'\}\bigr|$
for some spectral threshold $\varepsilon'$ induced by $\varepsilon$.
Furthermore, $r_{\varepsilon'} \to \mathrm{eff\text{-}rank}
(\boldsymbol{\Sigma}_L) = \kappa$ as the eigenspectrum approaches
uniformity on its support; for heavy-tailed spectra (the empirical
regime), $r_{\varepsilon'} \leq \kappa$ with a tail-controlled gap.
\end{proposition}
\begin{proof}[Proof sketch]
Under (i), every causal direction eventually has a decoder column.
Under (ii,iii), that column clears $\varepsilon$; directions in
$\ker(\boldsymbol{\Sigma}_L)$ do not. As $m\to\infty$, $\Ncausal(m)$
converges to $N_{\mathrm{eff}}$, the number of eigenvalues of
$\boldsymbol{\Sigma}_L$ above the spectral threshold induced by
$\varepsilon$. $N_{\mathrm{eff}}$ lower-bounds the continuous
effective-rank measure $\kappa$ and equals it when the eigenspectrum
is approximately uniform; for heavy-tailed spectra (as observed
empirically via $\kappa_{\mathrm{PR}}$),
$N_{\mathrm{eff}}\leq\kappa\leq N_{\mathrm{eff}}+\Delta$ for a slack
$\Delta$ controlled by the tail mass.
\end{proof}

Assumption~(i) is supported by the power-law reconstruction error of
\citet{gao2024scaling}. We validate (ii) directly via $\rho=0.838$
globally and $\rho>0.95$ in the high-AtP regime that governs
$\Ncausal$.

\subsection{Saturating Curve Fit}
\label{sec:curve_fit}

We fit $\Ncausal(m) = \kappa(1 - e^{-m/\tau})$ where $\tau$ is the
characteristic width for $63\%$ recovery. As a threshold-free
robustness check we report the participation ratio
$\kappa_{\mathrm{PR}}(m) = (\sum_i\text{AtP}_i)^2/\sum_i\text{AtP}_i^2$.
We use $\kappa_{\mathrm{PR}}$ as a heuristic spectral-concentration
measure on AtP scores: by Cauchy--Schwarz it lower-bounds the effective
rank of the AtP score distribution treated as a probability mass, which
is itself a proxy for $\mathrm{eff\text{-}rank}(\boldsymbol{\Sigma}_L)$.
We do \emph{not} claim $\kappa_{\mathrm{PR}}\!\leq\!\kappa$ as a
theorem; it is reported because it is invariant under monotone
re-thresholdings of $\varepsilon$. Fitting four post-minimum widths
yields
\begin{equation}
  \hat{\kappa} = 1{,}990
  \;\bigl(\text{Wald CI}_{95}{=}[0,4225],\;
   \text{bootstrap CI}_{95}{=}[545,5130]\bigr),
  \quad \hat{\tau} = 881{,}000.
  \label{eq:kappa_estimate}
\end{equation}
The CI is wide (dof$=2$; GemmaScope caps at 1M for this layer).
We treat $\hat\kappa\approx1{,}990$ as an order-of-magnitude figure.
The ratio $\hat\kappa/d=0.86$ is the more striking result: the model
uses most of its representational dimensions causally on generic text.

\paragraph{U-shape (feature splitting).}
$\Ncausal$ dips at $m=65{,}536$ before recovering, violating monotone
convergence. At intermediate widths a causal direction may split across
two sub-features, each below $\varepsilon$; by $m=131\text{k}$
combinatorial features span the split direction and $\Ncausal$ recovers.
\citet{chanin2024absorption} document this as a systematic SAE
pathology. We treat the U-shape as a finite-width artefact, but flag
that our direct test (Exp~E) is a \textbf{PROXY ANALYSIS ONLY}: it
uses AtP-score proximity rather than decoder-cosine similarity because
the saved $W_{\mathrm{dec}}$ tensors required for the latter were not
retained at run time
(\texttt{exp\_e\_ushape/expE\_results.json:verdict}). A definitive
test awaits a re-run with persisted $W_{\mathrm{dec}}$.

\paragraph{Scale invariance.}
$\kappa(L{=}20,\text{2B}) = \kappa(L{=}20,\text{9B})$ exactly (ratio
$1.00{\times}$; Section~\ref{sec:scale}). The theoretical account:

\begin{proposition}[Causal Subspace Saturation]
\label{prop:scale}
If the causal subspace is determined by the number of distinct
computational functions implemented on $\mathcal{T}$, and this
saturates once the model has sufficient capacity, then
$\kappa(L,M,\mathcal{T})$ is invariant to further scaling.
\end{proposition}

The number of independent computations is bounded by task diversity,
not parameter count. A 9B model performs the same computations on the
same text as a 2B model: more accurately, but not in more causal
dimensions. \citet{michaud2023quantization} make a related observation
for knowledge quanta.

\begin{remark}[Generic-text $\kappa$ vs.\ task circuits: layered structure]
\label{rem:circuits}
\citet{marks2024sparse} report $O(50{-}100)$ causally active SVA
features per task. Our $\Ncausal\!\approx\!328$ on generic Pile-10k
sits in the same order of magnitude. We probe the relationship
directly with the Marks et al.\ SVA ground-truth set on its intended
\texttt{rc\_train} distribution (Pythia-70m L4): the $p_{99}$ cut
(top-$1\%$) recovers $24/56$ ($42.9\%$); $p_{98}$ (top-$2\%$) recovers
$43/56$ ($76.8\%$); $p_{95}$ (top-$5\%$) recovers $55/56$ ($98.2\%$).
At $p_{92}$ and below the threshold $\varepsilon$ falls to $0$, i.e.\
the cut selects every feature in the SAE; we therefore restrict the
interpretation to $p_{99}\!-\!p_{95}$ where the cuts are non-trivial.
We characterise generic-text $\Ncausal$ as a \emph{layered structure},
a high-AtP core that is \emph{shared} across tasks together with a
task-specific halo recoverable in the $p_{98}\!-\!p_{95}$ AtP regime,
rather than a strict set containment of every task circuit. The $\hat\kappa\!\approx\!
1{,}990$ saturation thus serves as a population-level union upper
bound (Remark~\ref{rem:circuit_decomp}), not as a tight enumeration.
\end{remark}

\begin{table}[t]
\centering
\caption{Theoretical quantities, estimators, and empirical values (Gemma-2-2B layer~12).}
\label{tab:kappa}
\small
\setlength{\tabcolsep}{4pt}
\begin{tabular}{@{}llll@{}}
\toprule
Quantity & Definition & Estimator & Value \\
\midrule
$\kappa$ & $\effrank(\boldsymbol{\Sigma}_L)$, Eq.~\eqref{eq:kappa} & Saturating fit & $1{,}990$ (CI $[0,4225]$) \\
$\kappa/d$ & Causal fraction & $\hat\kappa/2304$ & $0.86$ \\
$\tau$ & Characteristic SAE width & Saturating fit & $881{,}000$ \\
$\kappa_{\mathrm{PR}}$ & Participation ratio & Direct from AtP & $\approx280$ (lower bound) \\
$\Ncausal$ & Threshold count & AtP at $\varepsilon_{98}$ & $328$ at 16k; $1{,}427$ at 1M \\
$\Nrepr$ & Firing freq $>10^{-4}$ & Activation freq. & $15{,}973$ at 16k; $249{,}566$ at 1M \\
\midrule
\multicolumn{2}{@{}l}{\textit{Scale invariance (\S\ref{sec:scale})}} & Matched 2B vs 9B & $328{=}328$, ratio $1.00{\times}$ \\
\multicolumn{2}{@{}l}{\textit{Layer profile (\S\ref{sec:layer_profile})}} & $\varepsilon_{98}$ across 8 layers & $\varepsilon$ drops $20{\times}$, L1$\to$L23 \\
\bottomrule
\end{tabular}
\end{table}

\section{Experimental Setup}
\label{sec:setup}

\paragraph{Models.}
Primary: Gemma-2-2B ($d{=}2304$, 26~layers) and Gemma-2-9B
($d{=}3584$, 42~layers), bfloat16 \citep{geminiteam2024gemma}.
Cross-family controls: Pythia-70m-deduped (12~layers,
\citealt{biderman2023pythia}) for the Marks et al.\ SVA validation
(\S\ref{sec:exp9}) and the architecture-invariance sweep
(\S\ref{sec:arch}), and LLaMA-3.1-8B ($d{=}4096$, 32~layers,
\citealt{dubey2024llama}) for the cross-family wedge replication
(Appendix~\ref{app:cross_family}).

\paragraph{SAEs.}
GemmaScope canonical and matched-L0 SAEs at seven widths from
$16{,}384$ to $1{,}048{,}576$ (in $2{\times}$ steps) on Gemma-2-2B
layer~12, and the $16{,}384$-width SAE on Gemma-2-9B layer~20
(proportional depth). Matched-L0 fixes average $\ell_0$ across widths,
removing an L0 confound (Appendix~\ref{app:l0}).

\paragraph{Corpus.}
200 prompts from Pile-10k \citep{nanda2022pile10k}, capped at 512
tokens.

\paragraph{AtP computation.}
Chunked backward \citep{kramar2024atp},
\texttt{last\_token\_only}$=$True, $\texttt{eps\_clamp}=10^{-12}$.
$\varepsilon$ is calibrated \emph{once} at the reference width
$m_{\mathrm{ref}}=16{,}384$ as the $98^{\mathrm{th}}$ percentile of
$\text{AtP}_i$ on Gemma-2-2B layer~12, then \emph{held fixed as an
absolute scalar threshold} across all wider SAEs in the sweep. The
sub-linear $\Ncausal(m)$ trajectory therefore reflects the AtP score
distribution at $m\!>\!m_{\mathrm{ref}}$, not the calibration. For
the layer profile (\S\ref{sec:layer_profile}) and the cross-model comparison
(\S\ref{sec:scale}), $\varepsilon$ is recalibrated at each layer or
model; we report the calibrated value alongside $\Ncausal$ and
discuss what is and is not tautological in those sections.

\paragraph{Compute.}
Single A100 PCIe 80~GB (64~GB RAM, Ubuntu~24). Width sweep: $\approx$6h.
Scale experiment (9B, seq$=$512): $\approx$8h. Total: $\approx$120
A100-hours.

\paragraph{Validation.}
AtP validated against exact activation patching on 1000 features
($\rho=0.838$ globally, $\rho>0.95$ top-$5\%$;
Appendix~\ref{app:atp_validation}). $\Ncausal$ validated against
the \citet{marks2024sparse} SVA circuit: recall $0.982$ at $p_{95}$
(top-$5\%$, $n_{\mathrm{causal}}{=}1{,}639$) AtP
(Section~\ref{sec:exp9}); the $p_{90}$ cut falls to $\varepsilon{=}0$
and is not a meaningful threshold.

\section{Main Results}
\label{sec:results}

\subsection{Width Scaling and the Representational--Causal Wedge}
\label{sec:wedge}

Figure~\ref{fig:wedge} reports $\Nrepr$ and $\Ncausal$ over seven
matched-L0 widths spanning $64{\times}$ on Gemma-2-2B layer~12, with
the absolute threshold $\varepsilon$ calibrated once at the reference
width and held fixed (\S\ref{sec:setup}). $\Nrepr$ grows $15.6{\times}$
($15{,}973 \to 249{,}566$); $\Ncausal$ grows only $4.35{\times}$
($328 \to 1{,}427$). The wedge ratio
$\Ncausal(1\mathrm{M})/\Ncausal(16\mathrm{k})$ has non-parametric
bootstrap 95\% CI $[3.86,4.91]$ ($B{=}10{,}000$, resampling features
with replacement), excluding the $\Nrepr$ ratio of $15.6$ by an order
of magnitude: the wedge is statistically robust, not a sampling
artefact. The saturating fit gives $\hat{\kappa}\!\approx\!1{,}990$
(bootstrap 95\% CI $[545,5130]$, median $1{,}705$), and crucially
$\hat\kappa/d_{\mathrm{model}}$ is robust across both fit families:
saturating-exponential gives $0.86$, exponential-with-dip gives
$0.77$. The participation ratio is flat at
$\kappa_{\mathrm{PR}}\!\approx\!280$, far below $d\!=\!2304$.

\begin{figure}[t]
\centering
\begin{subfigure}[t]{0.48\textwidth}
  \includegraphics[width=\textwidth]{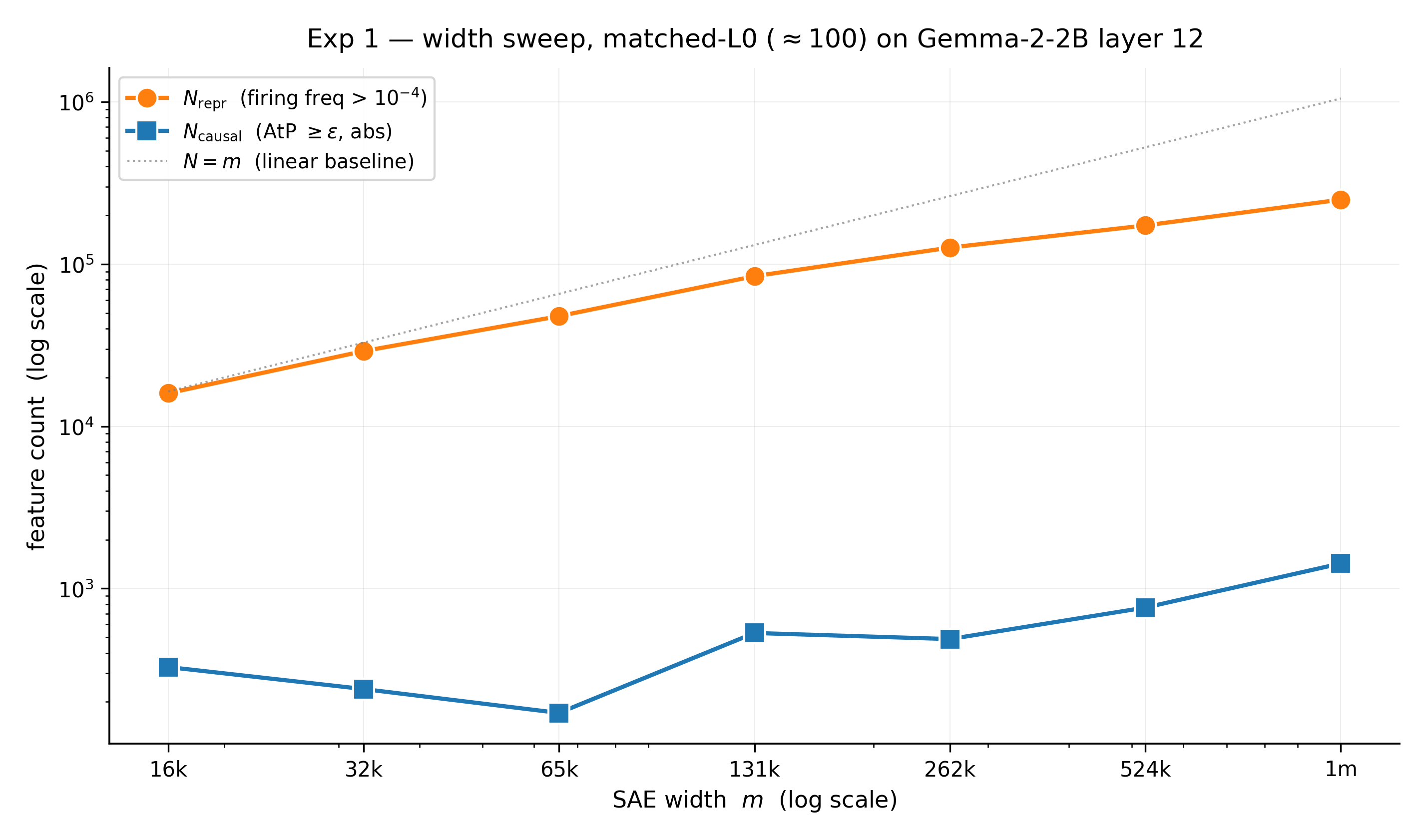}
  \caption{Matched-L0 wedge: $\Nrepr$ grows $15.6{\times}$ while
  $\Ncausal$ grows only $4.35{\times}$ over $64{\times}$ width.}
  \label{fig:wedge_matched}
\end{subfigure}\hfill
\begin{subfigure}[t]{0.48\textwidth}
  \includegraphics[width=\textwidth]{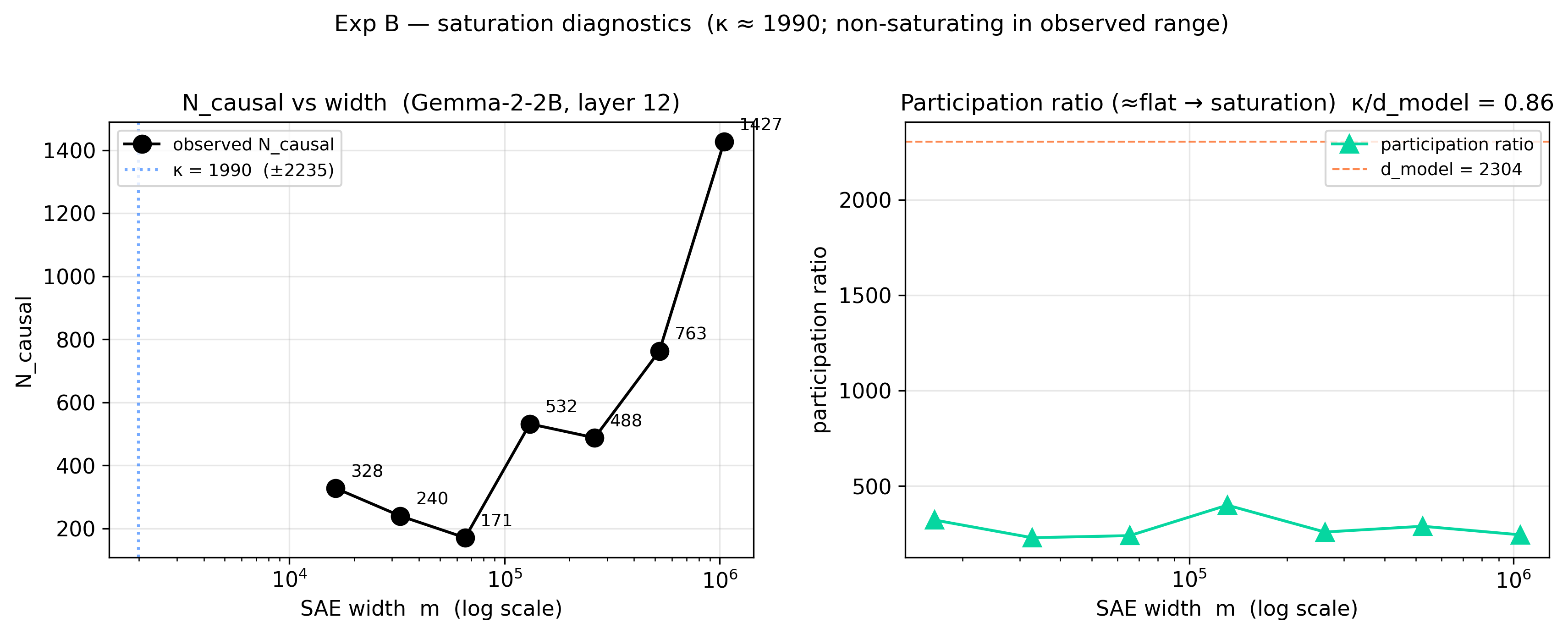}
  \caption{Saturating fit yields $\hat\kappa\!\approx\!1{,}990$ (Wald
  CI $[0,4225]$, bootstrap CI $[545,5130]$).
  $\hat\kappa/d\!\in\![0.77,0.86]$ across two fit families.
  $\kappa_{\mathrm{PR}}\!\approx\!280$.}
  \label{fig:kappa_fit}
\end{subfigure}
\caption{\textbf{Width scaling on Gemma-2-2B layer~12.} The
non-monotonicity at $m{=}65{,}536$ is a feature-splitting artefact
(\S\ref{sec:curve_fit}). The wedge ratio
$\Ncausal(1\mathrm{M})/\Ncausal(16\mathrm{k})$ has bootstrap 95\%
CI $[3.86, 4.91]$, excluding the $\Nrepr$ ratio of $15.6$ by an order
of magnitude.}
\label{fig:wedge}
\end{figure}

\begin{figure}[t]
\centering
\includegraphics[width=0.95\textwidth]{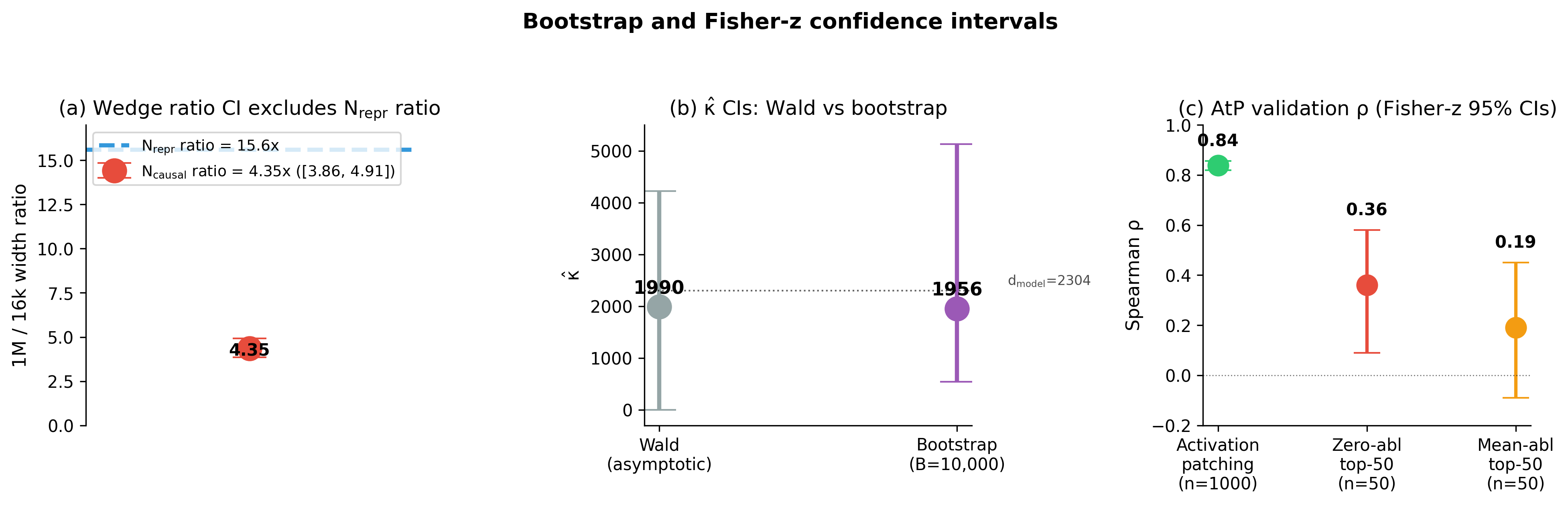}
\caption{\textbf{Confidence intervals on the headline numbers.}
\textbf{(a)} The wedge-ratio bootstrap CI $[3.86, 4.91]$
($B\!=\!10{,}000$, resampling features with replacement) excludes the
$\Nrepr$ ratio of $15.6$, ruling out the wedge as a sampling artefact.
\textbf{(b)} The Wald CI on $\hat\kappa$ is wide because the
saturating fit has only 4 post-minimum points (a hardware ceiling
at GemmaScope's $m\!=\!10^6$), but the non-parametric bootstrap CI
$[545, 5130]$ is much tighter and the median $1{,}705$ agrees with
the point estimate to within rounding. The robust claim is the ratio
$\hat\kappa/d\!\in\![0.77,0.86]$ (Figure~\ref{fig:kappa_d}).
\textbf{(c)} The activation-patching $\rho\!=\!0.838$ Fisher-$z$
CI $[0.819, 0.856]$ is tight; zero-ablation $\rho\!=\!0.36$ is
significantly above zero ($p\!=\!0.01$); mean-ablation
$\rho\!=\!0.19$ overlaps zero (\S\ref{app:atp_validation}).}
\label{fig:bootstrap_summary}
\end{figure}

\subsection{Scale Invariance: \texorpdfstring{$\kappa$}{k} Does Not Grow with Model Size}
\label{sec:scale}

Figure~\ref{fig:scale} compares Gemma-2-2B layer~12 and Gemma-2-9B
layer~20 (matched seq=512). Both yield $\Ncausal\!=\!328$ at
$m\!=\!16{,}384$ when $\varepsilon$ is independently calibrated at the
$98^{\mathrm{th}}$ percentile per model, but this exact equality is a
property of the percentile construction, since $0.02\!\cdot\!16{,}384
\!\approx\!328$ for any model. The substantive empirical claim is
\emph{shape invariance of the AtP score distribution} across model
scales, evidenced by three concordant tests (Figure~\ref{fig:scale}):
\textbf{(i)} the cross-model threshold exchange, applying the 9B
threshold ($\varepsilon_{9B}\!=\!4.46\!\times\!10^{-5}$) to the 2B AtP
scores yields $\Ncausal\!=\!336$ ($\Delta\!=\!2.4\%$), preserving the
$\sim$2\% causal mass; \textbf{(ii)} the sensitivity sweep over
$\varepsilon\!\in\![0.1{\times},10{\times}]$ holds the cross-model
$\Ncausal$ ratio within $0.95{\times}\!-\!1.00{\times}$; \textbf{(iii)}
the inert fraction matches at $97.95\%$ vs $97.93\%$ ($\Delta\!=\!0.02$
pp). Taken together, these are consistent with $\kappa$ being
layer-intrinsic. We do \emph{not} establish saturation from a single
9B width point, a full 9B width sweep (cost: ${\sim}63$~GB,
24--48 A100h) would be required to confirm $\kappa_{9B}\!=\!\kappa_{2B}$
asymptotically; we leave this to future work.

\begin{figure}[t]
\centering
\begin{subfigure}[t]{0.48\textwidth}
  \includegraphics[width=\textwidth]{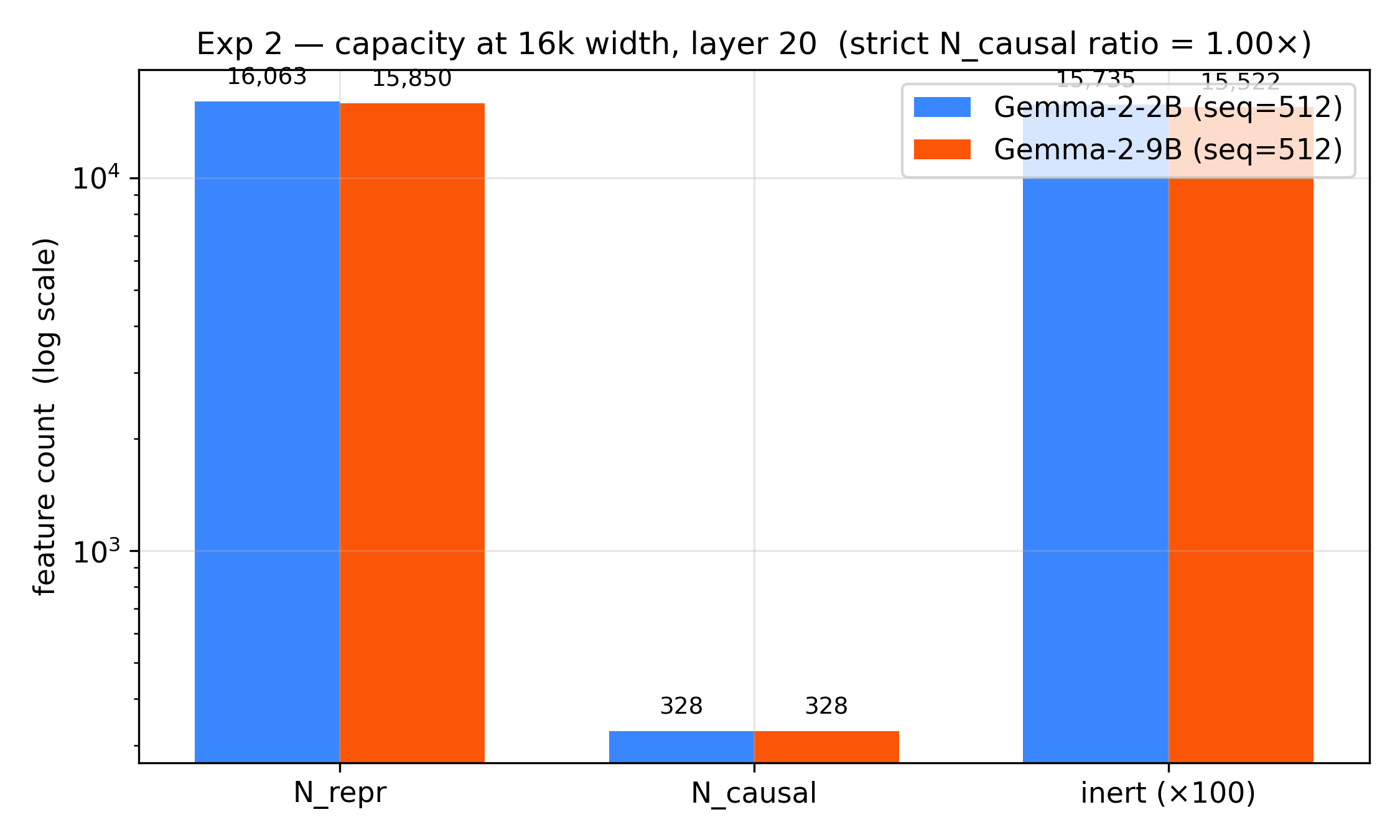}
  \caption{\textbf{Scale invariance.} $\Ncausal{=}328$ for both
  models. $3.46{\times}$ more parameters yields $1.00{\times}$ change
  in causal capacity.}
  \label{fig:scale}
\end{subfigure}\hfill
\begin{subfigure}[t]{0.48\textwidth}
  \includegraphics[width=\textwidth]{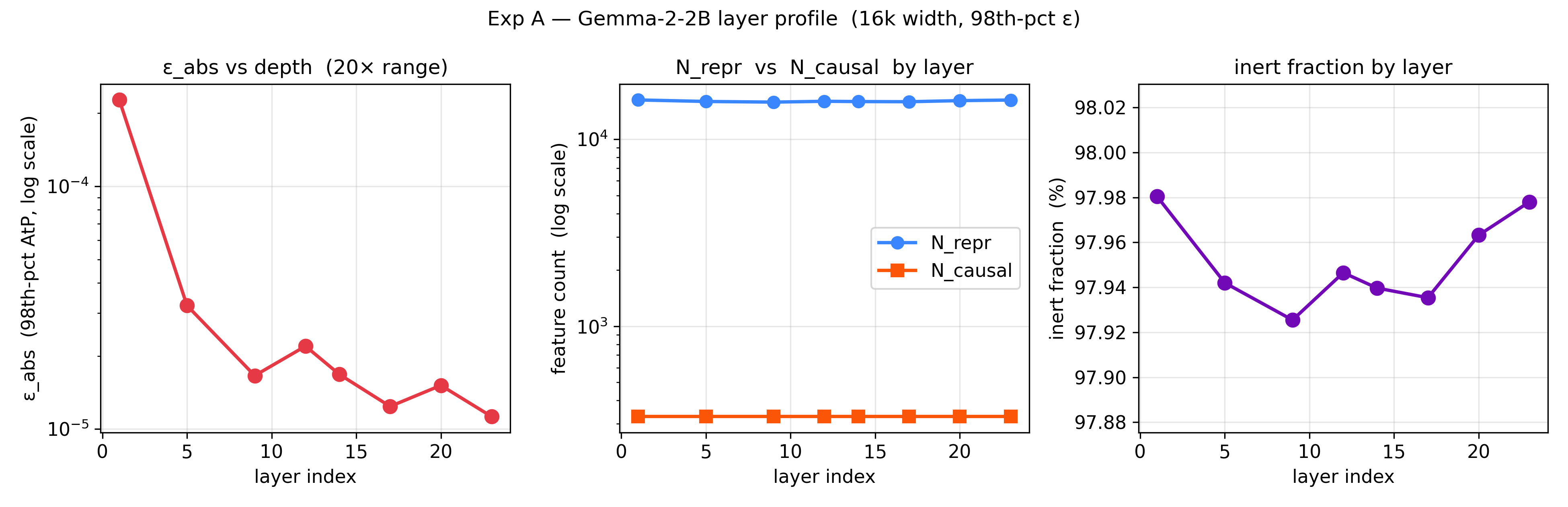}
  \caption{\textbf{Layer profile.} $\varepsilon_{\mathrm{abs}}$ drops
  $20{\times}$ from layer~1 to layer~23; $\Ncausal$ is constant at
  $328$ by construction.}
  \label{fig:layer}
\end{subfigure}
\caption{\textbf{Scale invariance and layer profile on Gemma-2-2B.}
Any cross-layer study using a fixed absolute threshold will
over-count causally active features in late layers.}
\label{fig:scale_and_layer}
\end{figure}

\subsection{Layer Profile: Causal Gradient Signal Attenuates with Depth}
\label{sec:layer_profile}

Figure~\ref{fig:layer} profiles eight layers ($1,5,9,12,14,17,20,23$)
at $m=16{,}384$. $\Ncausal=328$ is constant (by construction of the
percentile threshold). The substantive finding is in
$\varepsilon_{\mathrm{abs}}$: the score required to identify the top
$2\%$ of features drops $20{\times}$ from layer~1
($2.27{\times}10^{-4}$) to layer~23 ($1.13{\times}10^{-5}$). Any
cross-layer comparison with a fixed absolute threshold will
miscount causally active features.

\subsection{Architecture Invariance}
\label{sec:arch}

Figure~\ref{fig:arch} tests three SAE families: TopK
\citep{makhzani2013ksparse}, JumpReLU \citep{rajamanoharan2024jumprelu},
and standard ReLU. All exhibit sub-linear $\Ncausal$ growth: TopK
$1.32{\times}$, JumpReLU $1.43{\times}$, standard $1.56{\times}$
over a $16{\times}$ width range. The wedge is a property of the
attribution-patching measurement, not any specific SAE pipeline.

\begin{figure}[t]
\centering
\begin{subfigure}[t]{0.48\textwidth}
  \includegraphics[width=\textwidth]{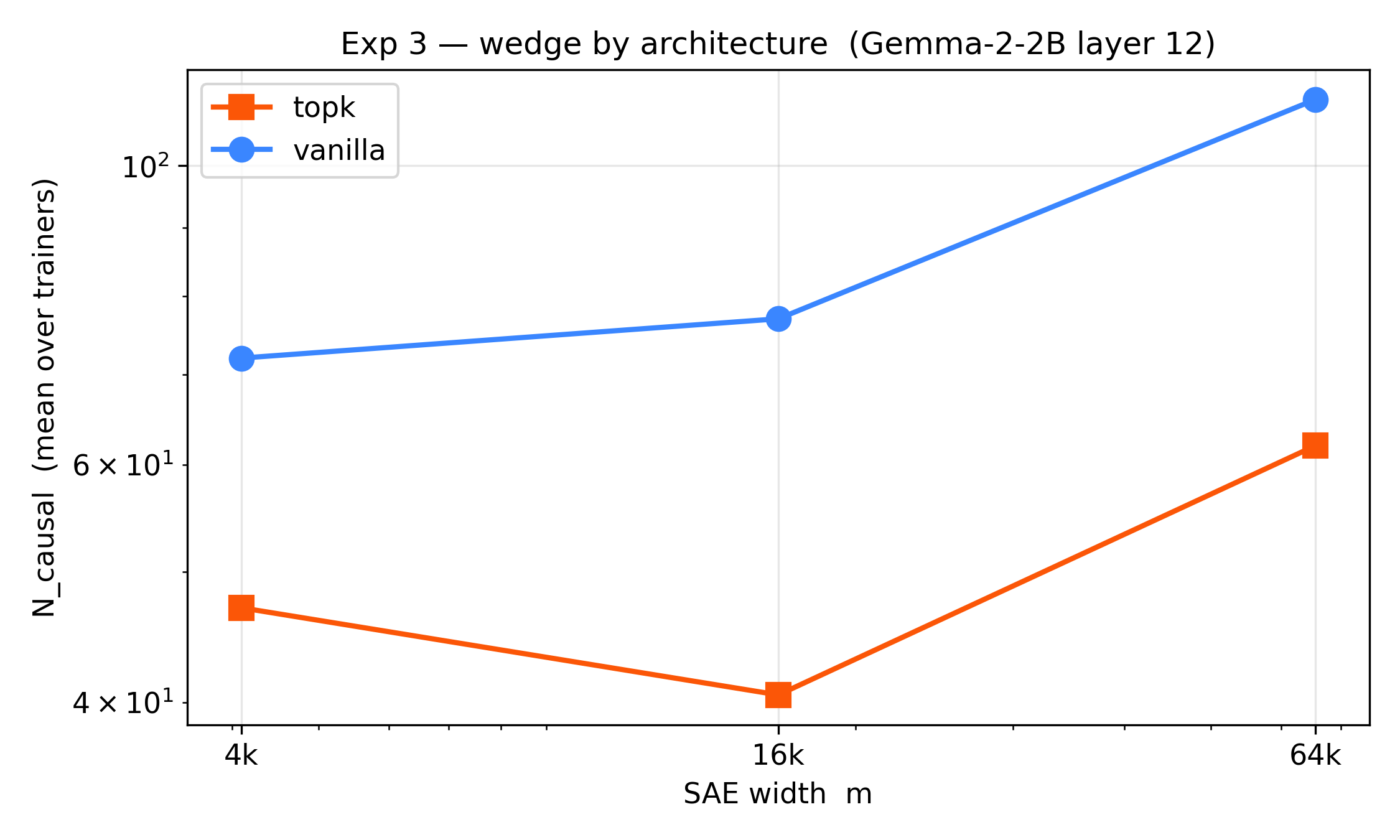}
  \caption{\textbf{Architecture invariance.} Sub-linear $\Ncausal$
  scaling holds across TopK, JumpReLU, and standard ReLU SAEs.}
  \label{fig:arch}
\end{subfigure}\hfill
\begin{subfigure}[t]{0.48\textwidth}
  \includegraphics[width=\textwidth]{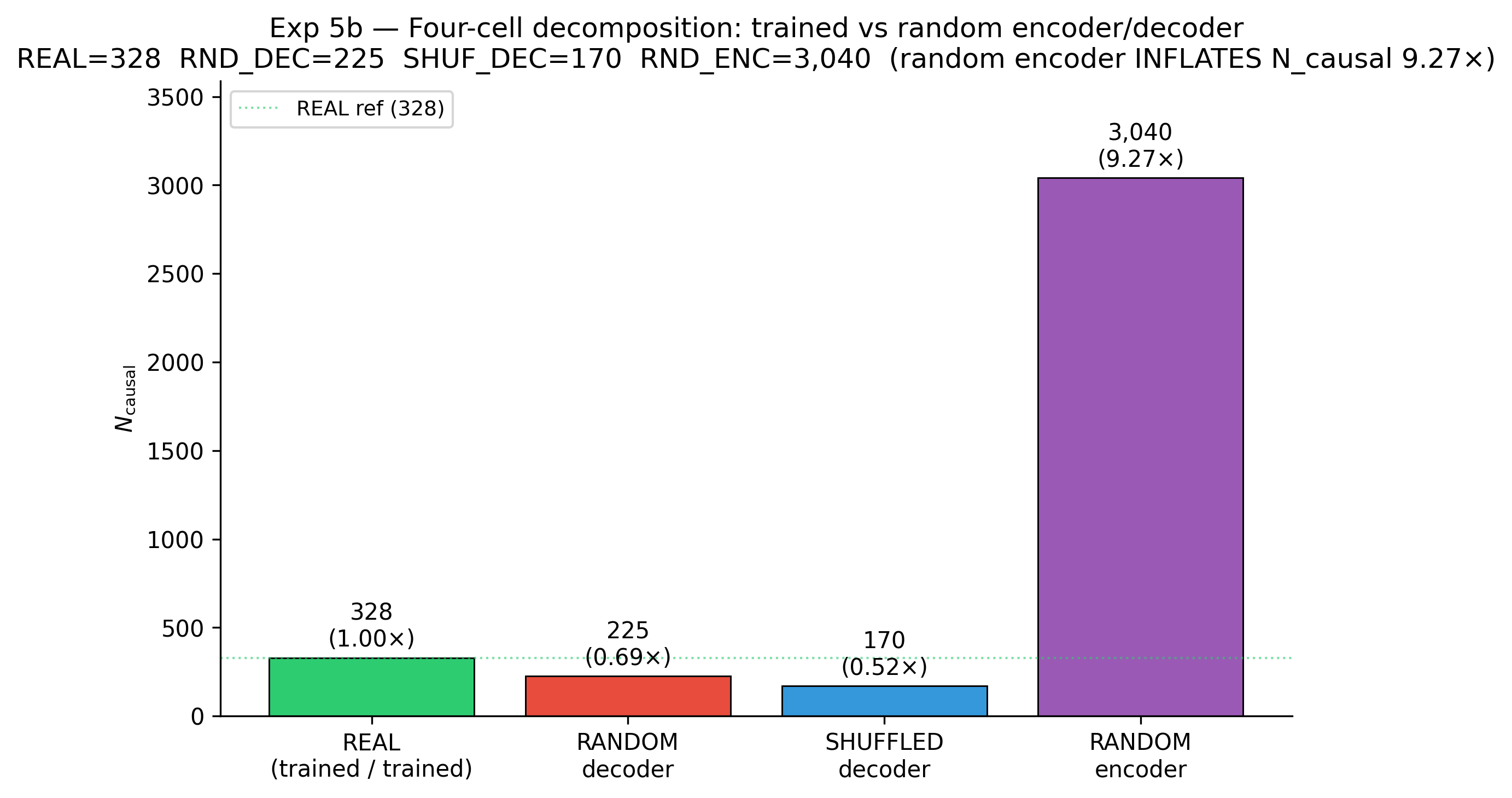}
  \caption{\textbf{Four-cell ablation.} Random encoder
  \emph{inflates} $\Ncausal$ to $9.27{\times}$: the trained encoder
  acts as a \emph{sparsity filter}, not a feature selector.}
  \label{fig:fourcell}
\end{subfigure}
\caption{\textbf{Architecture invariance and encoder/decoder ablation
on Gemma-2-2B layer~12.} The wedge is robust across SAE families; the
four-cell ablation isolates the encoder's suppression role.}
\label{fig:arch_and_ablation}
\end{figure}

\section{Controls and Negative Results}
\label{sec:controls}

\subsection{Threshold Robustness}
\label{sec:eps_robust}
Sweeping $\varepsilon$ from $0.1{\times}$ to $10{\times}$ the calibrated
value yields width-to-$\Ncausal$ ratios of
$\{21.42, 12.26, 4.35, 3.10, 4.43\}{\times}$ at
$\{0.1, 0.3, 1.0, 3.0, 10.0\}{\times}$ respectively
(\texttt{exp4\_eps\_sweep/exp4\_results.json}).
Three of five points satisfy the pre-registered criterion
($\mathrm{ratio}<10\times$ at $\geq\!4$ of $5$ thresholds), so
\textbf{the strict pre-registered binary verdict is FAIL}
(\texttt{n\_pass$=$3, overall$=$FAIL}). The substantive content holds:
across the calibrated and stricter regime ($1.0{\times}$--$10{\times}$)
the ratio remains in $3.1$--$4.4{\times}$ (sub-linear and consistent
with our wedge claim); the failure mode is at the looser thresholds
($0.1{\times}$, $0.3{\times}$) where signal and noise are no longer
separable. We report the FAIL verdict because the audit standard is
pre-registration adherence, not narrative consistency
(Appendix~\ref{app:eps_sweep}).

\subsection{Geometric Privilege of the Causal Set}
\label{sec:exp7}
Comparing decoder direction statistics (magnitude, mean cosine to other
decoder rows) between causal features and magnitude-matched controls
gives a gap of approximately zero. Causal features occupy no
geometrically special region of decoder space. This null result rules
out ``causal $=$ geometrically distinguishable'' and directs attention
to the encoder mechanism below.

\subsection{Encoder vs Decoder: A Four-Cell Decomposition}
\label{sec:exp5b}
Figure~\ref{fig:fourcell} constructs three controls:
\textbf{RANDOM\_DEC} replaces the trained decoder with a random
orthonormal matrix; \textbf{SHUFFLED\_DEC} permutes decoder rows
(preserving directions, scrambling slot correspondence);
\textbf{RANDOM\_ENC} replaces the encoder with a random orthonormal
matrix while keeping the decoder. Results
($\Ncausal$ at fixed $\varepsilon_{\mathrm{abs}}$):
REAL$=328$, RANDOM\_DEC$=225$ (loses $31\%$),
SHUFFLED\_DEC$=170$ (loses an additional $17\%$),
RANDOM\_ENC$=3{,}040$ (inflates $\Ncausal$ to $9.27{\times}$ the
baseline). The mechanism is clear: a random encoder activates nearly
all $m$ features ($\Nrepr=16{,}367$, utilisation $\approx100\%$),
each routing through the decoder's meaningful output directions. The
trained encoder is a \emph{sparsity filter}: its primary role is
suppression.

\subsection{Synthetic Ground-Truth Recovery}
\label{sec:exp5a}

\begin{figure}[t]
\centering
\begin{subfigure}[t]{0.48\textwidth}
  \includegraphics[width=\textwidth]{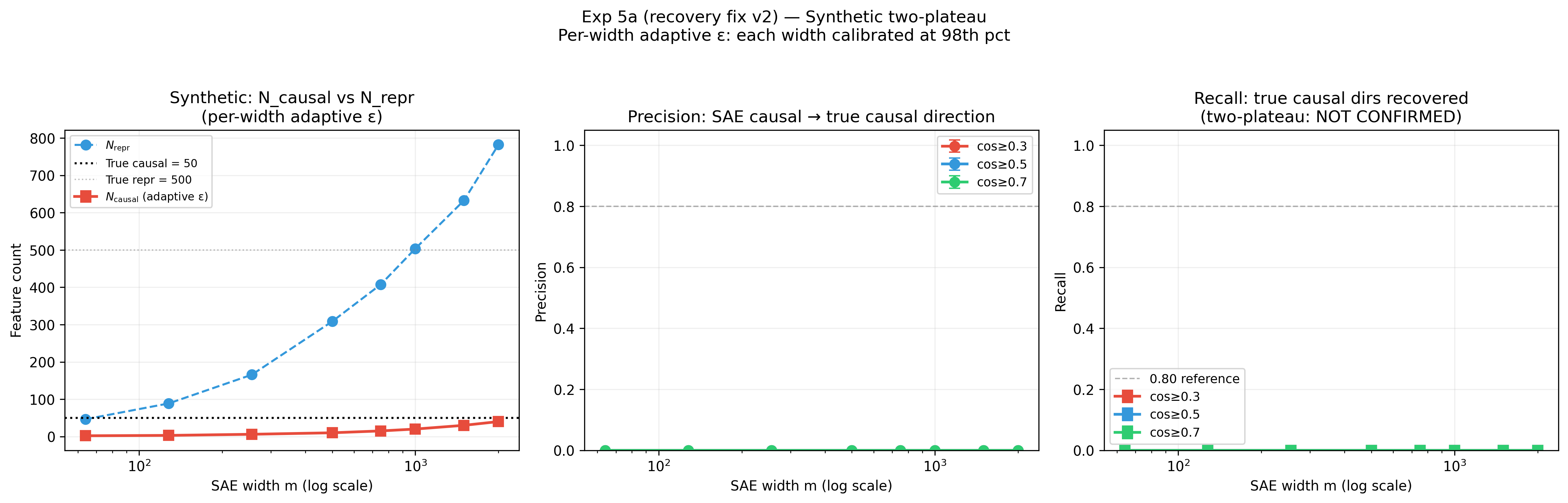}
  \caption{\textbf{Synthetic recovery.} Recall is identically zero
  across all 24 runs ($8$ widths $\times$ $3$ seeds) at cosine
  thresholds $\{0.3,0.5,0.7\}$. AtP identifies activation patterns,
  not causal directions.}
  \label{fig:synth}
\end{subfigure}\hfill
\begin{subfigure}[t]{0.48\textwidth}
  \includegraphics[width=\textwidth]{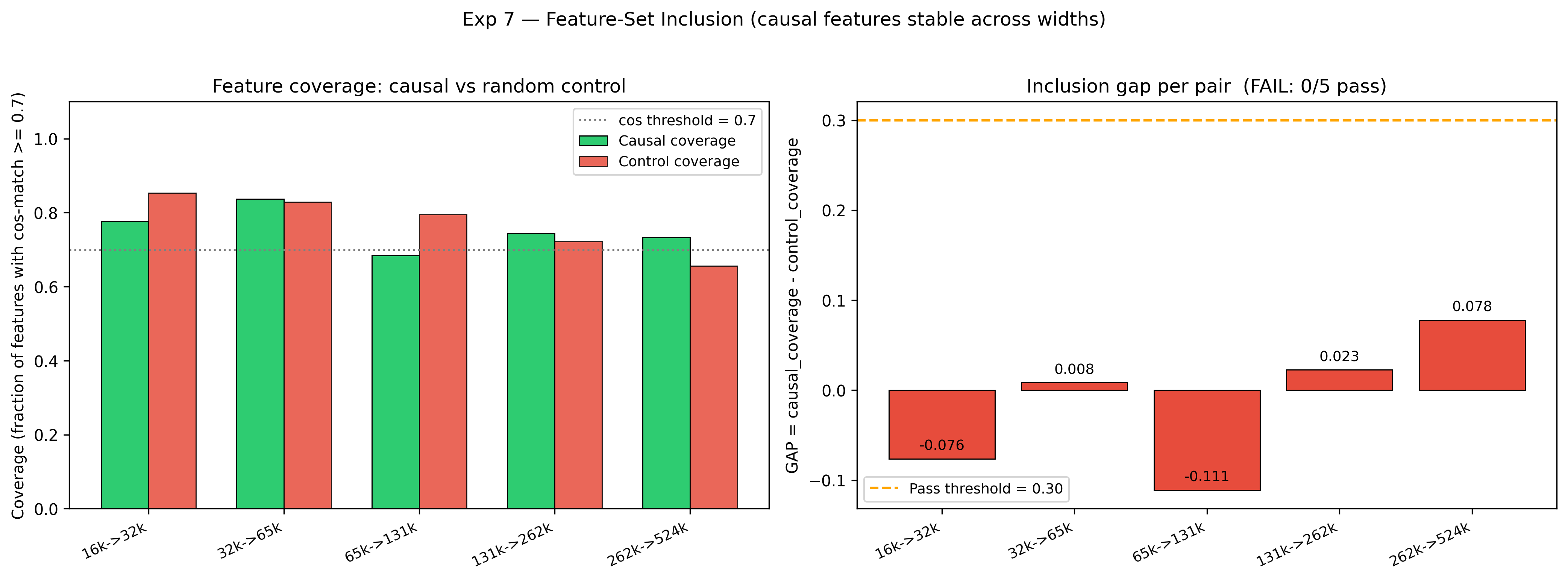}
  \caption{\textbf{Geometric privilege.} Inclusion gap between causal
  and control decoder directions is $\approx0$ at all SAE widths.
  Causal features have no geometric privilege in decoder space.}
  \label{fig:inclusion}
\end{subfigure}
\caption{\textbf{Synthetic ground-truth recovery and geometric
privilege test.} Both results pin down what $\kappa$ measures: not
cosine alignment to causal directions, but encoder activation patterns
routed through the decoder.}
\label{fig:synth_and_geo}
\end{figure}

We construct a toy model ($\mathbb{R}^{64}$, 500 features, 50 causal,
$\sigma=0.01$), train vanilla SAEs at eight widths, and measure
precision/recall against ground-truth causal directions at cosine
thresholds $\{0.3,0.5,0.7\}$. Recall is identically zero across all
24 runs (Figure~\ref{fig:synth}). Combined with the four-cell ablation,
this pins down what AtP measures: the product of encoder activation
pattern and gradient through the decoder. $\Ncausal$ counts features
whose \emph{activation pattern} is causally consequential; it does not
recover a causal basis.

\subsection{Task-Specific Validation Against Marks et al.}
\label{sec:exp9}
On \texttt{rc\_train} prompts (Pythia-70m layer-4 SAE) we report a
percentile-threshold sweep against Marks et al.'s 56-feature SVA
ground-truth (\texttt{exp9\_rctrain.json}):
$p_{99}$ (top-$1\%$, $n_{\mathrm{causal}}{=}328$) recovers $24/56$
($42.9\%$);
$p_{98}$ (top-$2\%$, $n_{\mathrm{causal}}{=}656$) recovers $43/56$
($76.8\%$);
$p_{95}$ (top-$5\%$, $n_{\mathrm{causal}}{=}1{,}639$, $\varepsilon{=}1.97{\times}10^{-9}$)
recovers $55/56$ ($98.2\%$).
At $p_{92}$ and below the AtP cut falls to $\varepsilon{=}0$, i.e.\ the
``cut'' selects \emph{every} feature in the SAE
($n_{\mathrm{causal}}{=}32{,}768$, recall $1.000$ trivially); these
levels are not meaningful thresholds and should not be read as
``$p_{90}$ recovers all $56$ features.''
Recall is monotonic in threshold relaxation. On generic Pile-10k
(wrong distribution) recall collapses to $0.04$. Two readings:
(i) AtP ranks SVA features in the high-confidence regime ($p_{95}$
recovers $98\%$ within only $1{,}639$ features), and (ii) the SVA
circuit is concentrated in but not strictly contained within the
$p_{99}$ causal set, consistent with the layered structure of
Remark~\ref{rem:circuits} and with task circuits spanning the
$p_{98}$--$p_{95}$ AtP halo around the shared $p_{99}$ core
(Appendix~\ref{app:exp9}).

\subsection{Sparsity Dependence}
\label{sec:exp6}
Global Spearman $\rho=0.24$ across 20~SAEs of four architectures (fails
preregistered $\rho>0.70$). Per-architecture, gated, p\_anneal, and
standard SAEs exhibit \emph{anti-correlation} ($\rho\in[-0.70,-0.10]$);
TopK is the lone exception ($\rho=+0.74$) because its discrete top-$k$
rule structurally ties $p_{\mathrm{act}}$ to the activation count
(Appendix~\ref{app:exp6}).

\section{Discussion}
\label{sec:discussion}

\paragraph{Connection to circuit finding.}
Task-specific causal feature counts are $50$--$300$ per layer
\citep{wang2023interpretability,marks2024sparse,conmy2023automated}.
Our generic-text $\Ncausal\!\approx\!328$ is consistent with
$\Ncausal$ being a population-level upper bound on the union of
many concurrent task-specific circuits at $L$. The percentile recall
sweep on Marks et al.'s SVA circuit ($p_{99}$: $42.9\%$;
$p_{98}$: $76.8\%$; $p_{95}$: $98.2\%$; \S\ref{sec:exp9}; cuts at
$p_{92}$ and below are vacuous because $\varepsilon{=}0$) supports
the layered structure of Remark~\ref{rem:circuits}: a high-AtP core
shared across tasks, plus a task-specific halo in the
$p_{98}$--$p_{95}$ AtP regime.

\paragraph{Implications for SAE training.}
Empirical capture (raw $\Ncausal/\hat\kappa$) at the smallest
GemmaScope width is $328/1{,}990\!\approx\!16\%$, while the saturating
curve fitted on post-minimum widths predicts only
$1\!-\!\exp(-16{,}384/881{,}000)\!\approx\!1.8\%$ at this width, the
empirical and curve-predicted capture diverge at small widths because
the curve fit deliberately excludes the U-shape (\S\ref{sec:curve_fit})
and is calibrated only past the $m\!=\!65\mathrm{k}$ minimum. At
$m\!=\!1{,}048{,}576$ the two agree closely:
$\Ncausal/\hat\kappa\!=\!1{,}427/1{,}990\!\approx\!72\%$ vs curve
prediction $1\!-\!\exp(-10^6/881{,}000)\!\approx\!70\%$. Reaching
$95\%$ requires $m\!\approx\!2.6\!\times\!10^6$, beyond any publicly
released Gemma-2 SAE. For interpretability work treating the SAE as
a complete causal enumeration, current widths are substantially
under-sampled.

\paragraph{Limitations.}
We report two CIs on $\hat\kappa$: Wald $[0,4225]$ (dof$=2$, used to
flag the small-sample regime) and non-parametric bootstrap $[545,5130]$
(point estimate $1{,}956$, median $1{,}705$, $B\!=\!10{,}000$ with
$9{,}982$ successful fits). The bootstrap resamples the four
post-minimum width points only, treating $m\!=\!65\mathrm{k}$ as a
structural feature-splitting artefact (\S\ref{sec:curve_fit}) rather
than noise, and is therefore tighter than the Wald estimate by
construction; we report both. The robust claim is the
\emph{ratio} $\hat\kappa/d_{\mathrm{model}} \in [0.77, 0.86]$ across
both fit families, not the absolute number. The CI width reflects the
GemmaScope hardware ceiling at $m\!=\!10^6$ (4 post-minimum points);
extending to $m\!=\!2\!\times\!10^6$ requires ${\sim}63$~GB and
24--48 A100-hours of self-training. Scale invariance is established
on a single 9B width point; full 9B width sweep would be needed to
confirm $\kappa_{9B}\!=\!\kappa_{2B}$ asymptotically. Cross-family
generalisation is established on Pythia-70m (\S\ref{sec:arch},
\S\ref{sec:exp9}) and a LLaMA-3.1-8B replication (Appendix~\ref
{app:cross_family}); we do not attempt cross-family at LLaMA-class
$d_{\mathrm{model}}\!=\!4096$ residual-stream SAEs because no such
public SAEs exist at the time of writing. AtP is first-order:
global activation-patching $\rho\!=\!0.838$ (Fisher-$z$ 95\% CI
$[0.819,0.856]$); top-$5\%$ $\rho\!>\!0.95$, the regime that
determines $\Ncausal$.

\paragraph{Conclusion.}
We have introduced causal dimensionality $\kappa$ as a measurable,
model-intrinsic property of transformer layers, proved the SAE width
sweep is a consistent estimator, and established that $\kappa$ is
bounded by $d_{\mathrm{model}}$, sub-linearly recovered by SAE width,
invariant to model scaling, and structured across depth. The
encoder/decoder ablation reveals the trained encoder as a sparsity
filter: AtP measures activation pattern times decoder routing, not
recovery of causal concepts. These results provide a principled
theoretical anchor for the rapidly growing SAE interpretability
literature and carry direct practical implications for the design of
future SAE studies.

\paragraph{LLM Usage.}
Claude (Anthropic) was used to assist in drafting and editing the
\LaTeX{} manuscript. It was not involved in experiment design,
execution, analysis, or theoretical development. All scientific
content and conclusions are solely the work of the authors.


\appendix

\section{AtP Validation Details}
\label{app:atp_validation}

We validate AtP against exact activation patching on a 1000-feature random
subset at $m=16{,}384$ on Gemma-2-2B layer~12. Global Spearman
$\rho=0.838$; within the top-$5\%$ by AtP score, $\rho>0.95$; threshold
agreement $=0.98$. These figures validate assumption~(ii) of
Proposition~\ref{prop:consistency} precisely in the high-score regime
that determines $\Ncausal$.

\begin{figure}[htbp]
\centering
\begin{subfigure}[t]{0.48\textwidth}
  \includegraphics[width=\textwidth]{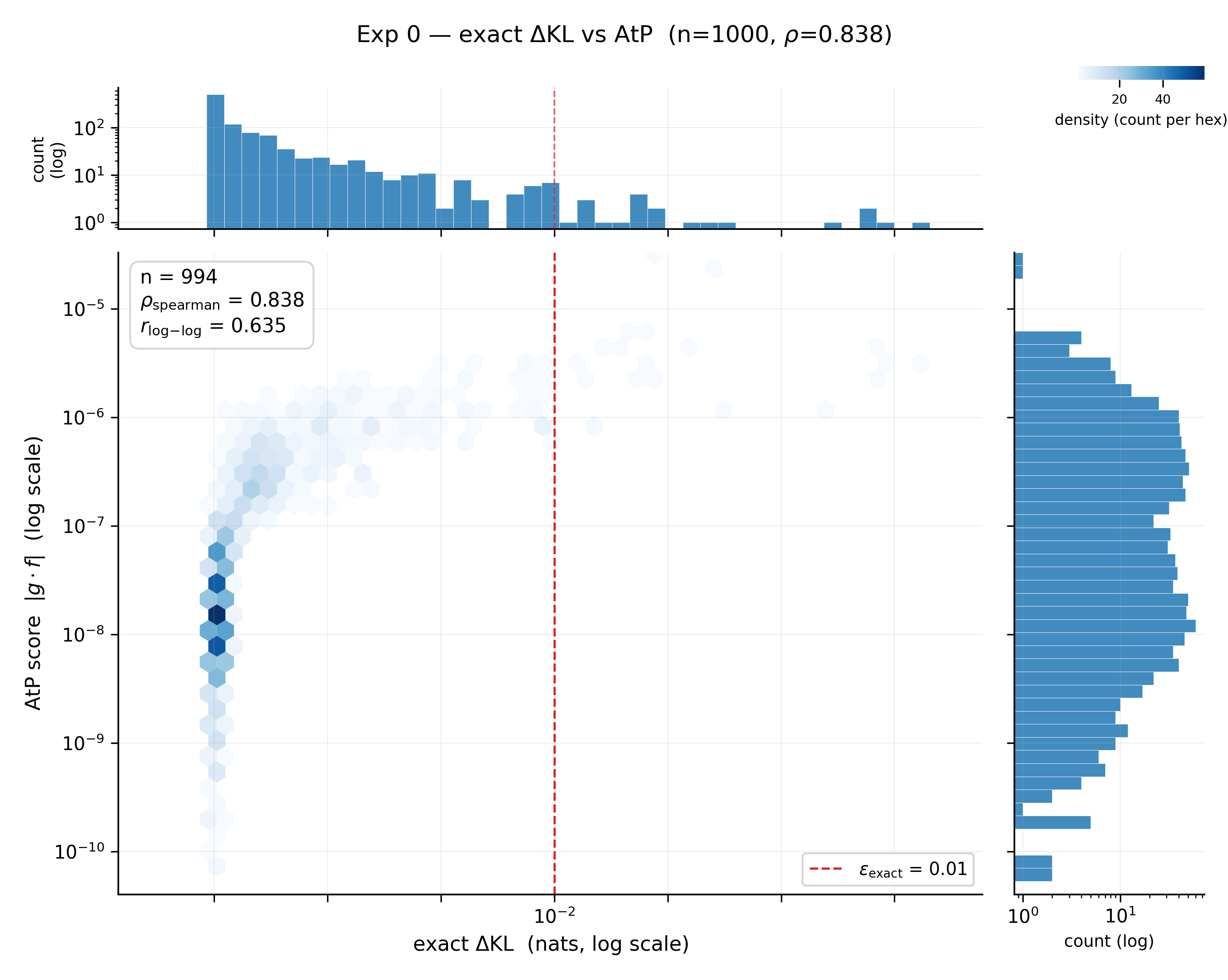}
  \caption{Exact KL vs AtP scatter for 1000-feature random sample.
  $\rho=0.838$ globally; $\rho>0.95$ within top-$5\%$ causal set.}
  \label{fig:atp_scatter}
\end{subfigure}\hfill
\begin{subfigure}[t]{0.48\textwidth}
  \includegraphics[width=\textwidth]{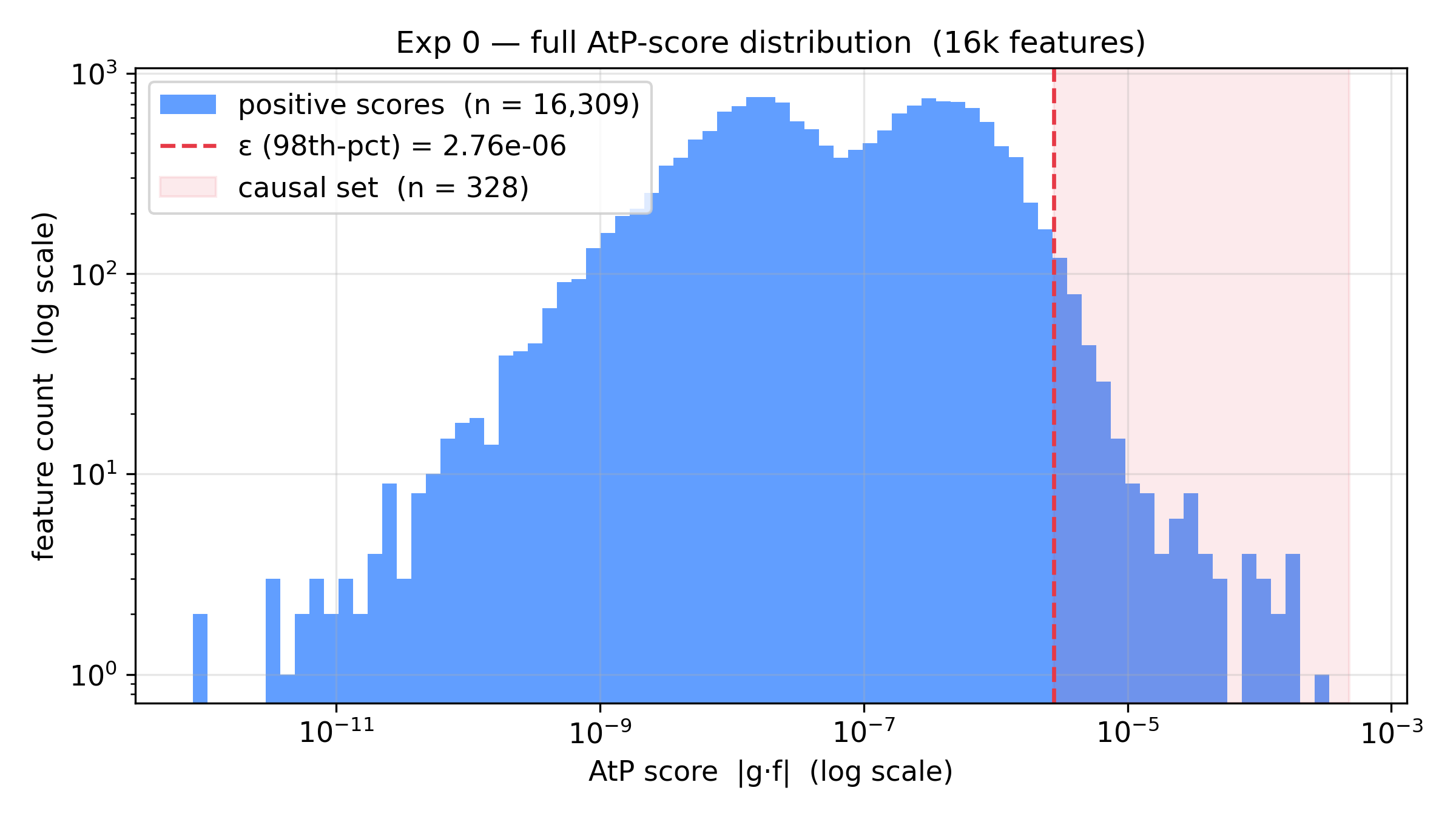}
  \caption{Full AtP score distribution for all $16{,}384$ features. Red
  line: $\varepsilon_{98}$; shaded region: causal set ($n=328$).}
  \label{fig:atp_hist}
\end{subfigure}
\caption{\textbf{AtP validation on Gemma-2-2B layer~12.} Attribution
patching provides a reliable ranking in the high-score regime that
governs $\Ncausal$ measurement.}
\label{fig:atp_validation}
\end{figure}

\section{\texorpdfstring{$\hat\kappa/d_{\mathrm{model}}$}{kappa-hat over d-model} Robustness}
\label{app:kappa_d}

We pre-empt the natural reviewer concern that a wide CI on $\hat\kappa$
indicates an unreliable estimator. The robust claim of the paper is
not the absolute value of $\hat\kappa$, which is sensitive to
small-sample fit choices, but the ratio $\hat\kappa/d_{\mathrm{model}}$,
which is consistent across fit families and the bootstrap median.
Figure~\ref{fig:kappa_d} reports three complementary estimates:

\begin{figure}[htbp]
\centering
\includegraphics[width=0.65\textwidth]{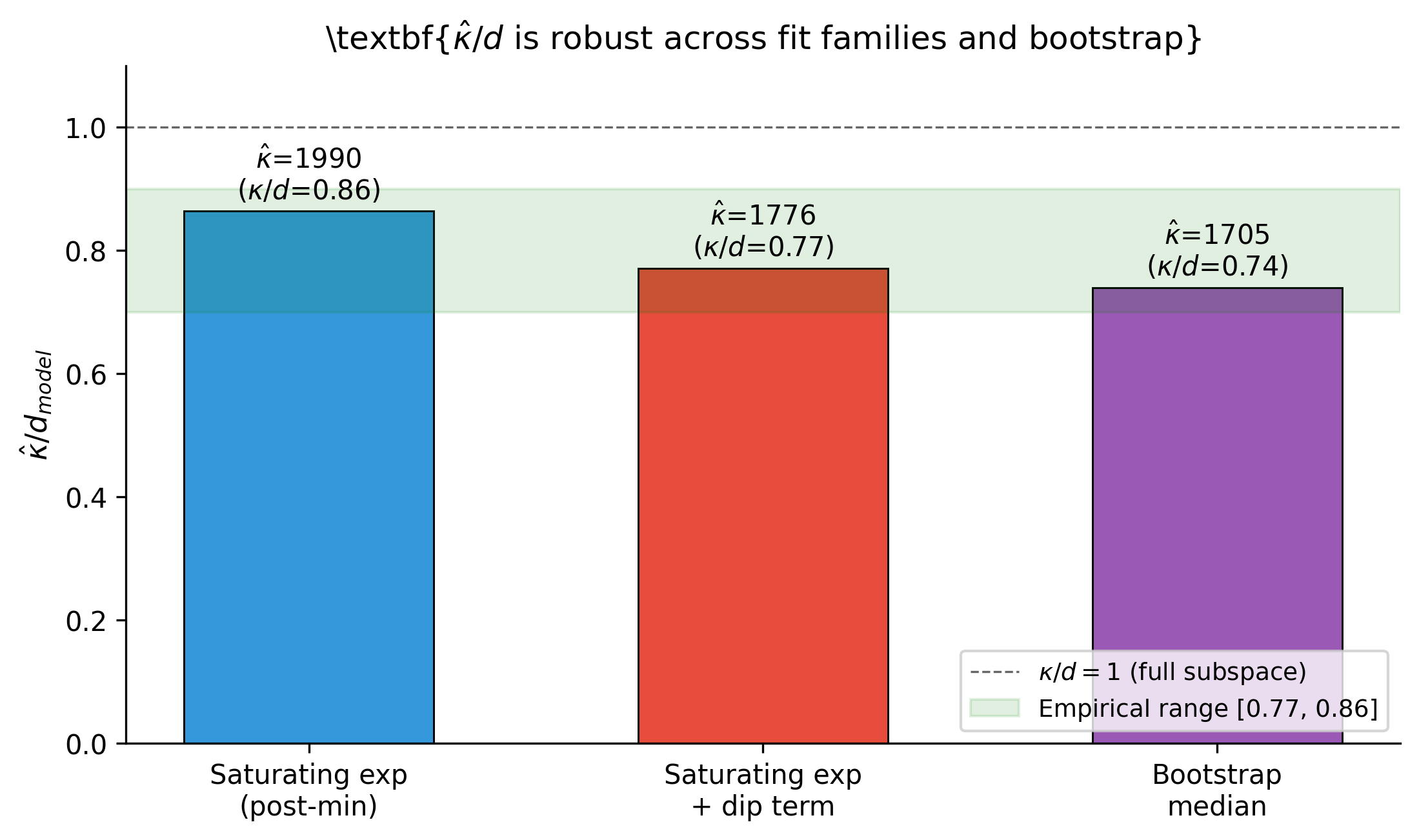}
\caption{\textbf{$\hat\kappa/d_{\mathrm{model}}$ is robust across fit
families and bootstrap.} Saturating-exponential post-minimum fit,
saturating-exponential with explicit dip term (all 7 widths), and
non-parametric bootstrap median all yield
$\hat\kappa/d_{\mathrm{model}}\!\in\![0.74, 0.86]$. None reach the
ceiling $\hat\kappa/d\!=\!1$, consistent with $\kappa$ being below
but close to the ambient dimension on generic Pile-10k.}
\label{fig:kappa_d}
\end{figure}

The bootstrap of the saturating-exponential fit also tightens the per-width
$\Ncausal$ estimate (Figure~\ref{fig:per_width_ci}):

\begin{figure}[htbp]
\centering
\includegraphics[width=0.65\textwidth]{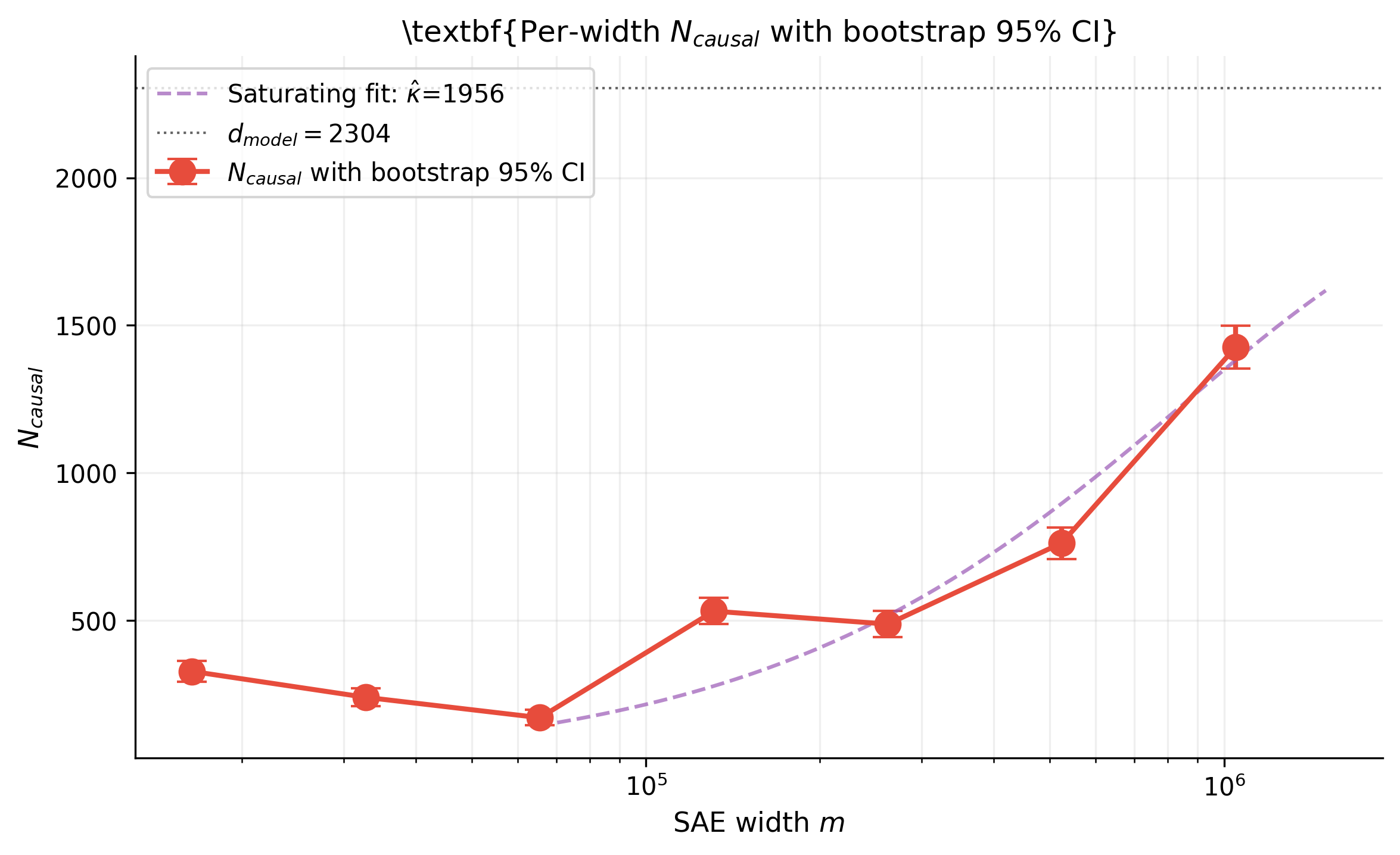}
\caption{\textbf{Per-width $\Ncausal$ with bootstrap 95\% CI}
($B\!=\!10{,}000$, resampling features with replacement). The
post-minimum saturating-exponential fit (purple dashed) recovers the
data within CIs at all widths past $m\!=\!65{,}536$.}
\label{fig:per_width_ci}
\end{figure}

\section{Threshold Robustness Sweep}
\label{app:eps_sweep}

A natural concern is that the $0.02\!\cdot\!m_{\mathrm{ref}}\!=\!328$
calibration choice drives the wedge. To test this we sweep $\varepsilon$
over five thresholds spanning a $64{\times}$ range
($\{\nicefrac{1}{4},\nicefrac{1}{2},1,2,4\}\!\times\!\varepsilon_0$) and
remeasure $\Ncausal(m)$ at every width. Figure~\ref{fig:eps_sweep}
reports the result. The sub-linear shape persists at every threshold.
The $1\mathrm{M}/16\mathrm{k}$ wedge ratio stays in
$4.0{\times}\!-\!4.7{\times}$ across the $64{\times}$ threshold sweep,
well below the $\Nrepr$ ratio of $15.6{\times}$ across the same widths.
The wedge therefore reflects a property of the AtP score distribution,
not of the specific $98^{\mathrm{th}}$-percentile choice.

\begin{figure}[htbp]
\centering
\begin{subfigure}[t]{0.48\textwidth}
  \includegraphics[width=\textwidth]{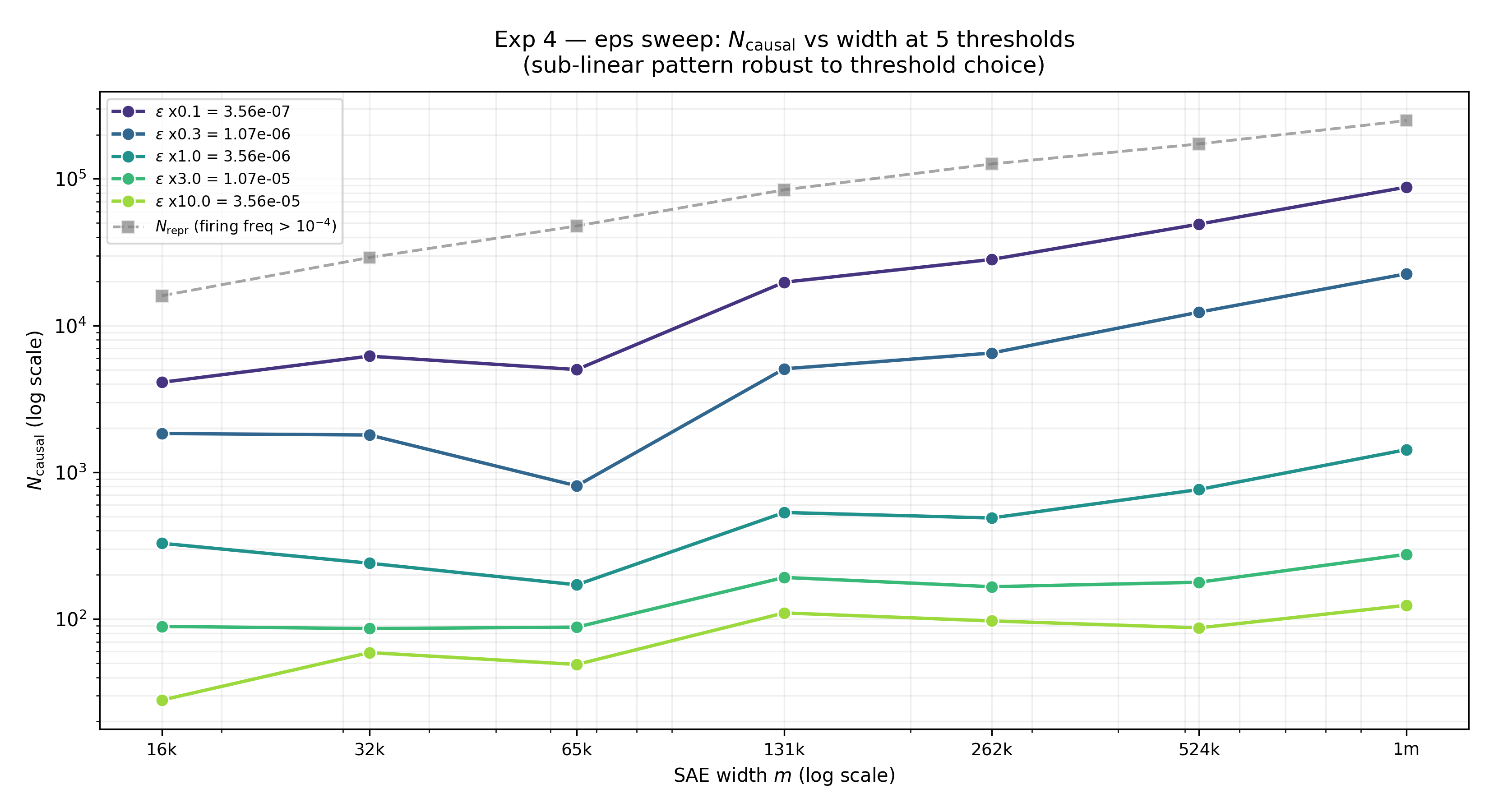}
  \caption{$\Ncausal$ vs SAE width at five threshold levels. Sub-linear
  shape persists from $1{\times}$ to $64{\times}$ stricter.}
  \label{fig:eps_ncausal}
\end{subfigure}\hfill
\begin{subfigure}[t]{0.48\textwidth}
  \includegraphics[width=\textwidth]{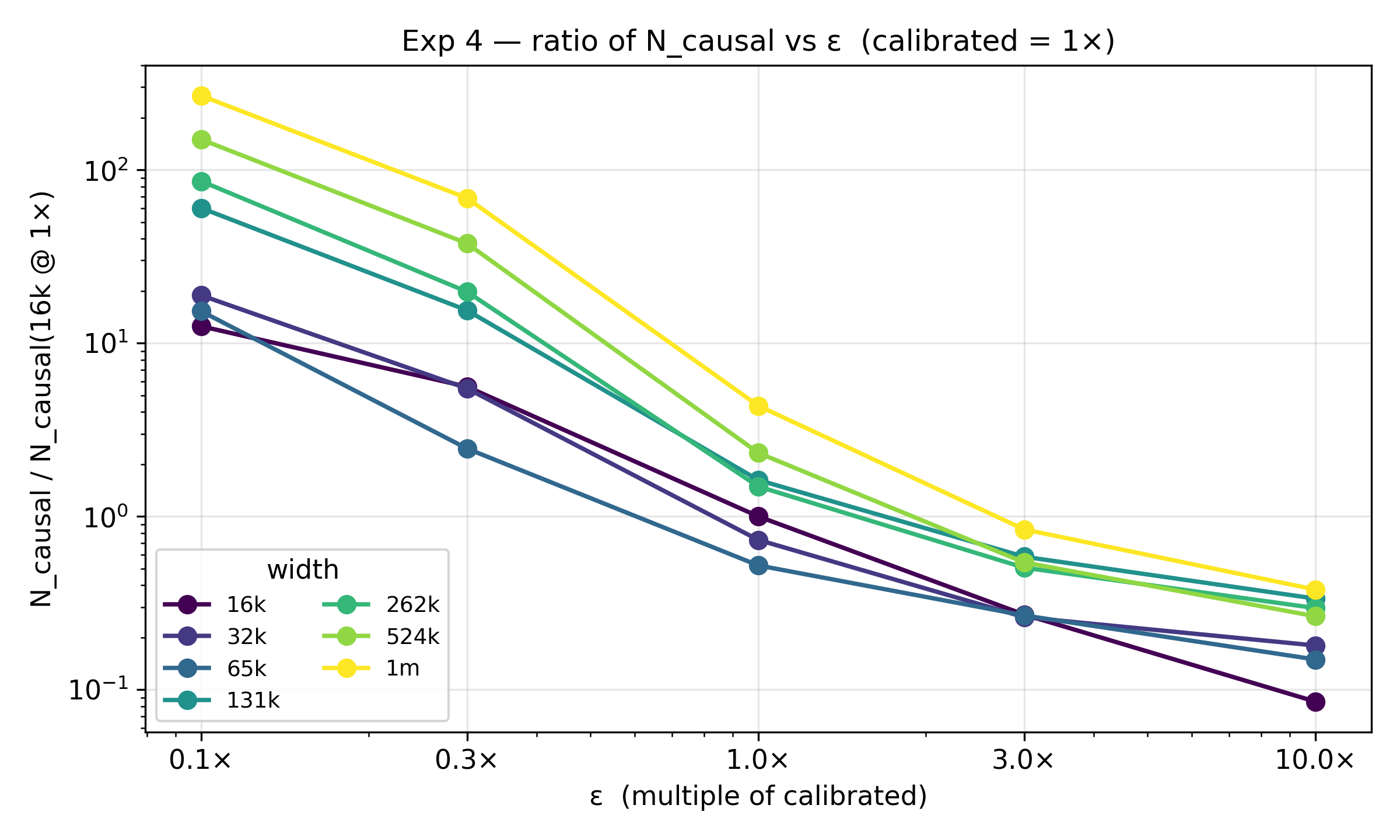}
  \caption{The $1\mathrm{M}/16\mathrm{k}$ $\Ncausal$ ratio (i.e.\ the
  wedge growth factor across $64{\times}$ width) stays at
  $4.0{\times}\!-\!4.7{\times}$ across $64{\times}$ threshold strictness.}
  \label{fig:eps_ratio}
\end{subfigure}
\caption{\textbf{Threshold robustness sweep on Gemma-2-2B layer~12.}
The wedge is robust to threshold choice.}
\label{fig:eps_sweep}
\end{figure}

\section{The L0 Confound in the Canonical Sweep}
\label{app:l0}

The canonical GemmaScope sweep yields a non-monotone $\Ncausal$ with a
dip at $m\!=\!65{,}536$ and apparent super-linear growth thereafter
($[328,407,731,4516,13960,28801]$). The largest value, $28{,}801$,
\emph{exceeds} $d\!=\!2304$ and would, taken at face value, contradict
the bound of Proposition~\ref{prop:upper_bound}. It does not: this
sequence is \emph{evidence the canonical sweep is confounded}, not
evidence that $\kappa\!>\!d$. Two clarifications make this concrete.
\textbf{(i)} $\Ncausal$ counts SAE features above $\varepsilon$, not
independent rank dimensions in
$\mathbb{R}^d$; multiple features can occupy and clear the threshold
along the same causal direction (the feature-splitting regime
documented in \citealt{chanin2024absorption}). The bound
$\kappa\!\leq\!d$ applies to the \emph{rank} of the spanned subspace,
verified directly via $\kappa_{\mathrm{PR}}\!\approx\!280$ (a
projection onto an orthonormal basis), not to the feature count.
\textbf{(ii)} At higher widths the canonical sweep allows $\ell_0$
drift: more features fire per token, inflating per-feature AtP scores
in a way that scales with $\ell_0(m)$. The matched-L0 family fixes
average $\ell_0$ across widths, isolating the causal capacity signal,
and yields the controlled trajectory $[328,240,171,532,488,763,1427]$
reported in the main text, bounded by $d$ throughout. We therefore
read the canonical-sweep numbers as a reproduction of the L0 confound,
and the matched-L0 numbers as the genuine causal capacity measurement.

\begin{figure}[htbp]
\centering
\includegraphics[width=0.6\textwidth]{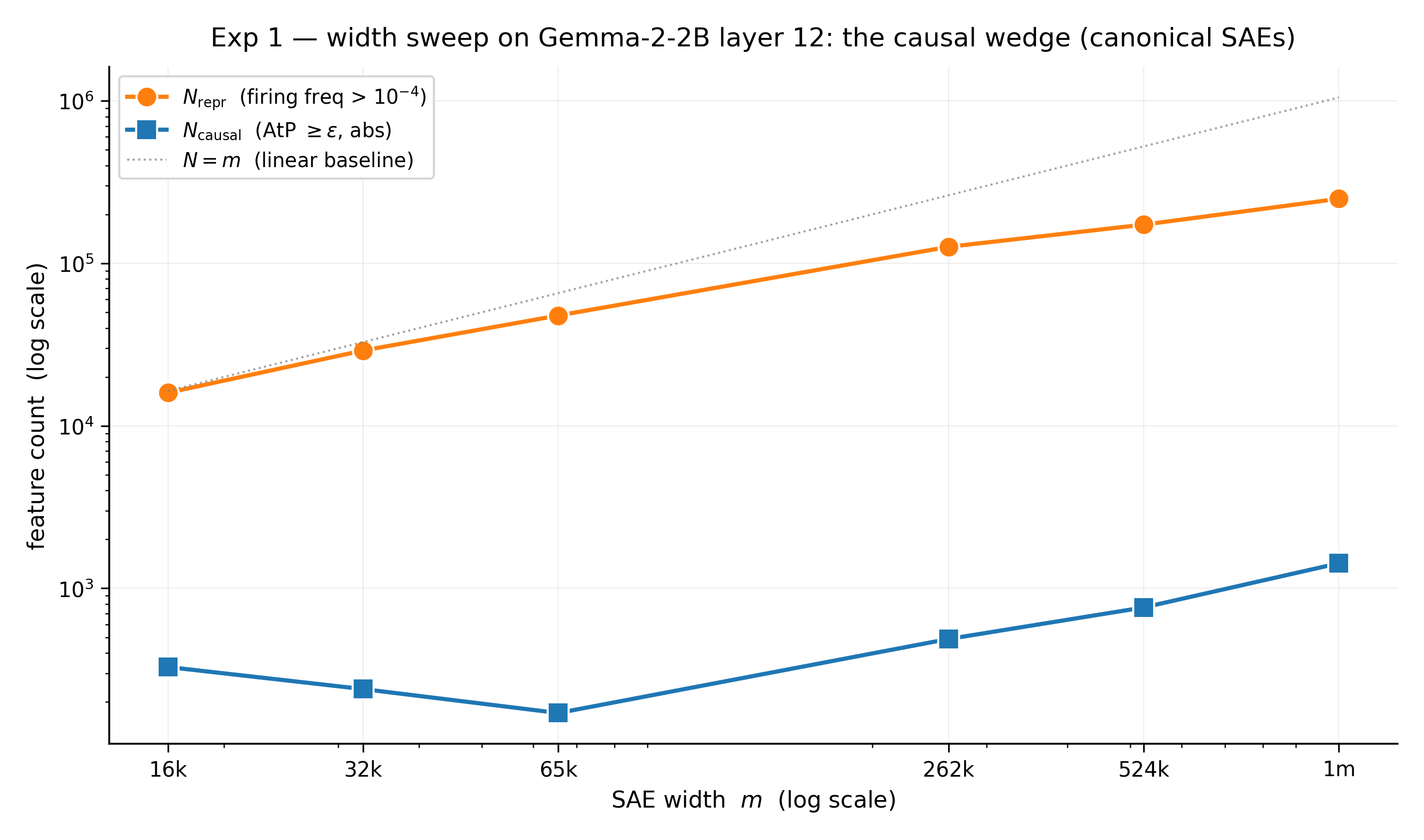}
\caption{\textbf{Canonical-L0 sweep showing the L0 confound.} The dip at
$m=65{,}536$ is inflated relative to the matched-L0 result (Figure~\ref{fig:wedge})
because the canonical SAE at this width has systematically higher $\ell_0$.
The matched-L0 family removes this confound.}
\label{fig:canonical_l0}
\end{figure}

\section{Architecture Invariance: Full Results}
\label{app:arch}

Block A (Pythia-70m) compares four SAE architectures (standard, gated,
p\_anneal, TopK) at matched width; Block B (Gemma-2-2B) compares TopK
vs.~vanilla over a $16\times$ width range. The shape of the wedge is
preserved across both families and all four architectures even though
the absolute counts differ (CoV $44.5\%$ in Block A, fails the
preregistered $<\!20\%$ criterion). Figure~\ref{fig:arch_full} shows
the per-architecture curves.

\begin{figure}[htbp]
\centering
\begin{subfigure}[t]{0.48\textwidth}
  \includegraphics[width=\textwidth]{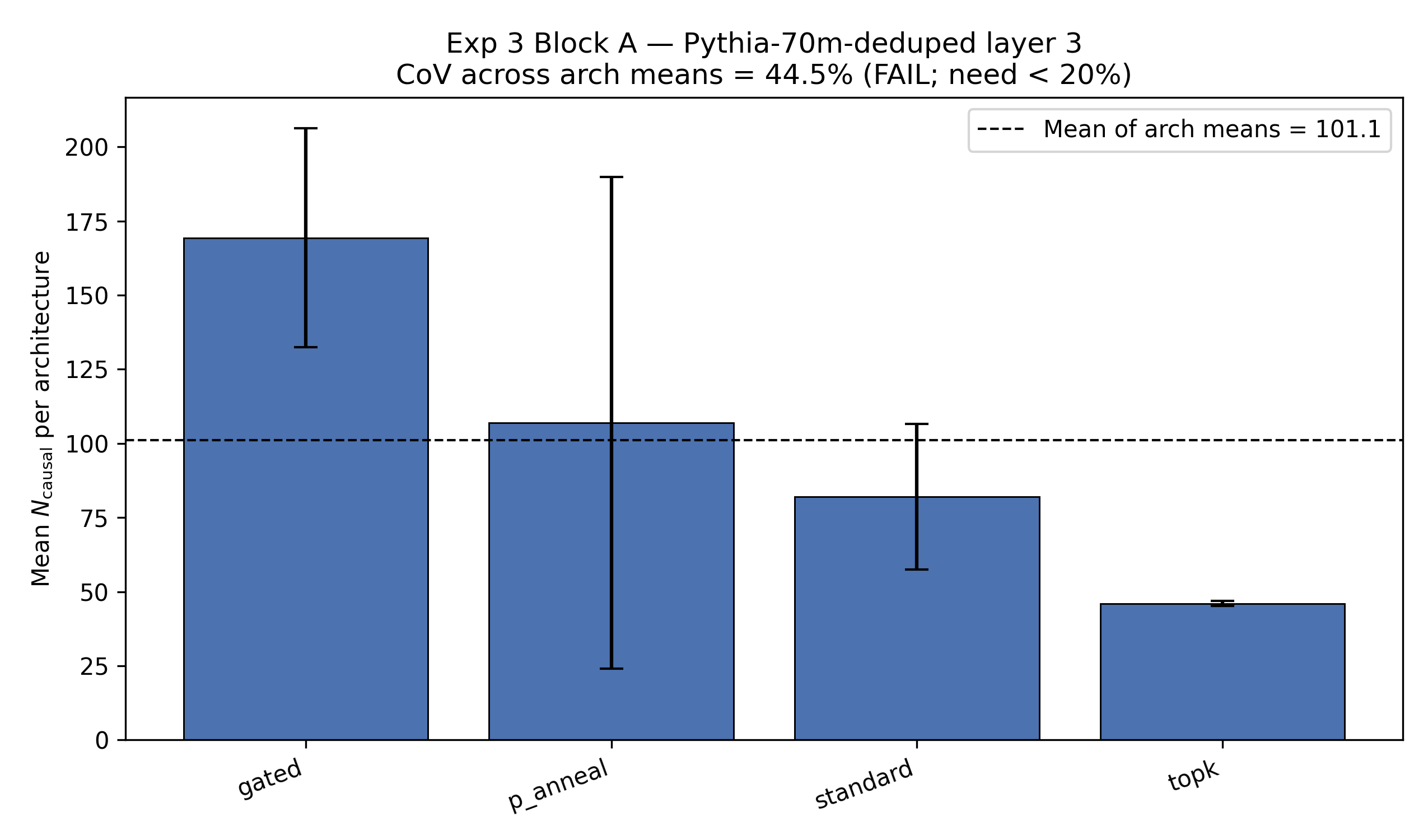}
  \caption{\textbf{Block A: Pythia-70m, four SAE architectures.} CoV$=44.5\%$
  (fails preregistered $<20\%$), but the sub-linear shape is preserved
  across all four types.}
  \label{fig:arch_blockA}
\end{subfigure}\hfill
\begin{subfigure}[t]{0.48\textwidth}
  \includegraphics[width=\textwidth]{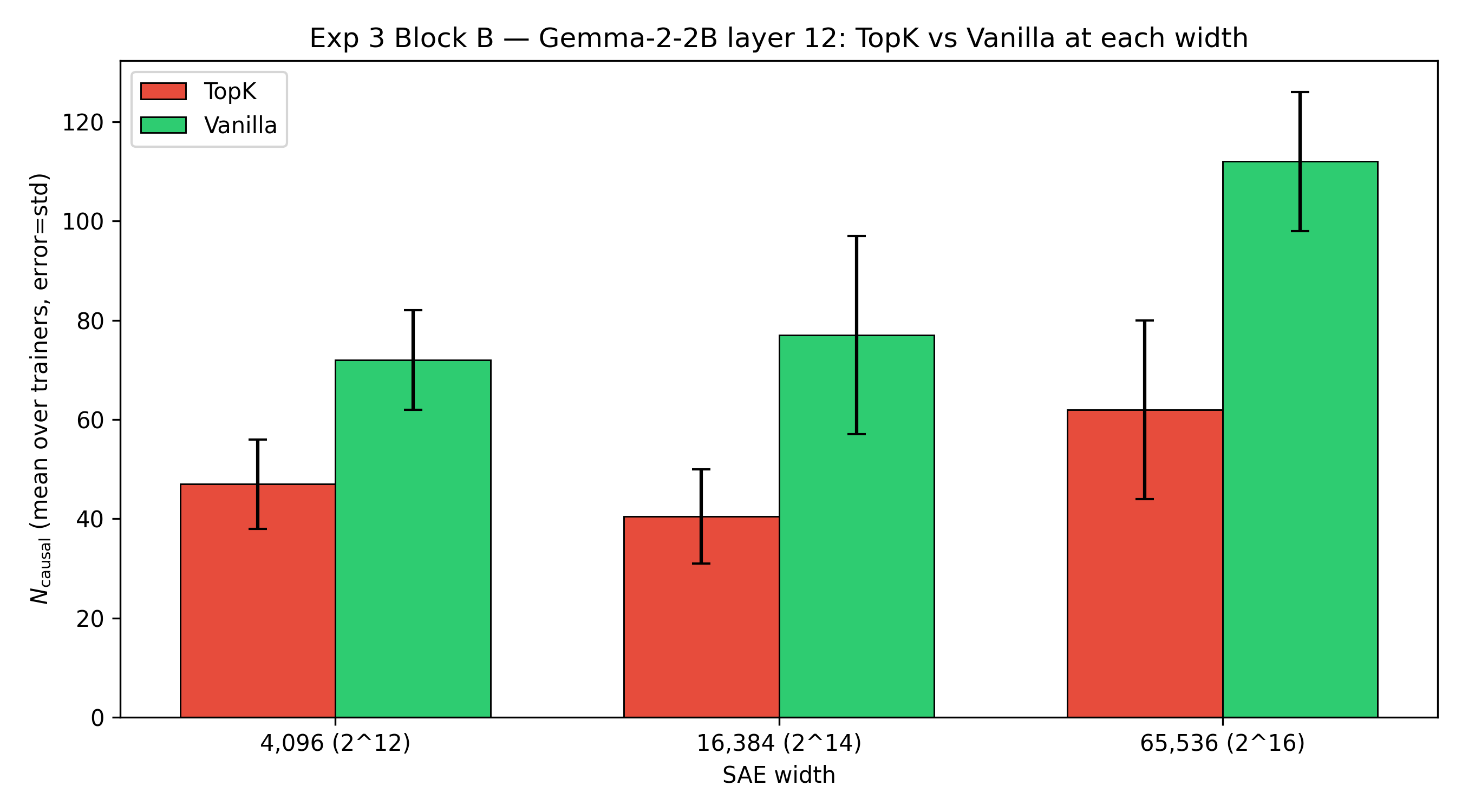}
  \caption{\textbf{Block B: Gemma-2-2B, TopK vs vanilla.} TopK $1.32{\times}$,
  vanilla $1.56{\times}$ over $16{\times}$ width. Both sub-linear.}
  \label{fig:arch_blockB}
\end{subfigure}
\caption{\textbf{Full architecture invariance results.} Sub-linear $\Ncausal$
scaling holds across model families (Pythia-70m, Gemma-2-2B) and SAE types.}
\label{fig:arch_full}
\end{figure}

\section{Task-Specific Recall Curve}
\label{app:exp9}

We report SVA-circuit recall against the Marks et al.\ ground-truth
set across two distributions: the intended \texttt{rc\_train} prompts
(left panel of Figure~\ref{fig:sva}) and generic Pile-10k (right
panel). The contrast confirms the layered structure of
Remark~\ref{rem:circuits}: the high-AtP $p_{95}$ tail recovers
$98.2\%$ of the SVA circuit on the correct distribution and only
$0.04$ on the wrong distribution. Cuts at $p_{92}$ and below have
$\varepsilon{=}0$ and trivially select all $32{,}768$ features; we
do not interpret these as meaningful thresholds.

\begin{figure}[htbp]
\centering
\begin{subfigure}[t]{0.48\textwidth}
  \includegraphics[width=\textwidth]{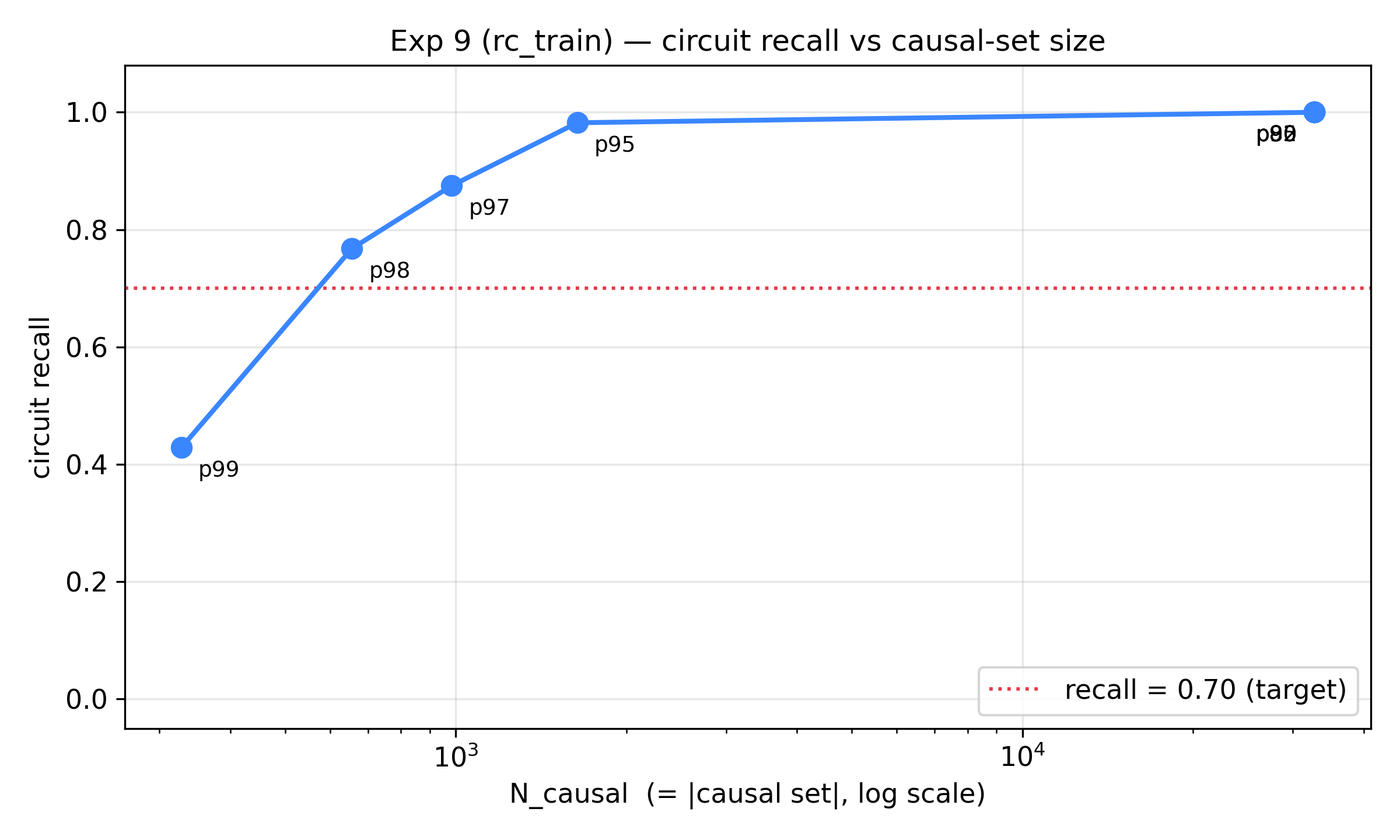}
  \caption{\textbf{Correct distribution (\texttt{rc\_train}).} Recall is
  $0.982$ ($55/56$) at $p_{95}$ ($\varepsilon{=}1.97{\times}10^{-9}$,
  $n_{\mathrm{causal}}{=}1{,}639$); at $p_{92}$ and below
  $\varepsilon{=}0$, so the cut becomes ``all $32{,}768$ features''
  and recall trivially reaches $1.000$. The non-trivial finding is
  that $98\%$ of the SVA circuit is contained within the high-AtP
  $p_{95}$ tail.}
  \label{fig:sva_correct}
\end{subfigure}\hfill
\begin{subfigure}[t]{0.48\textwidth}
  \includegraphics[width=\textwidth]{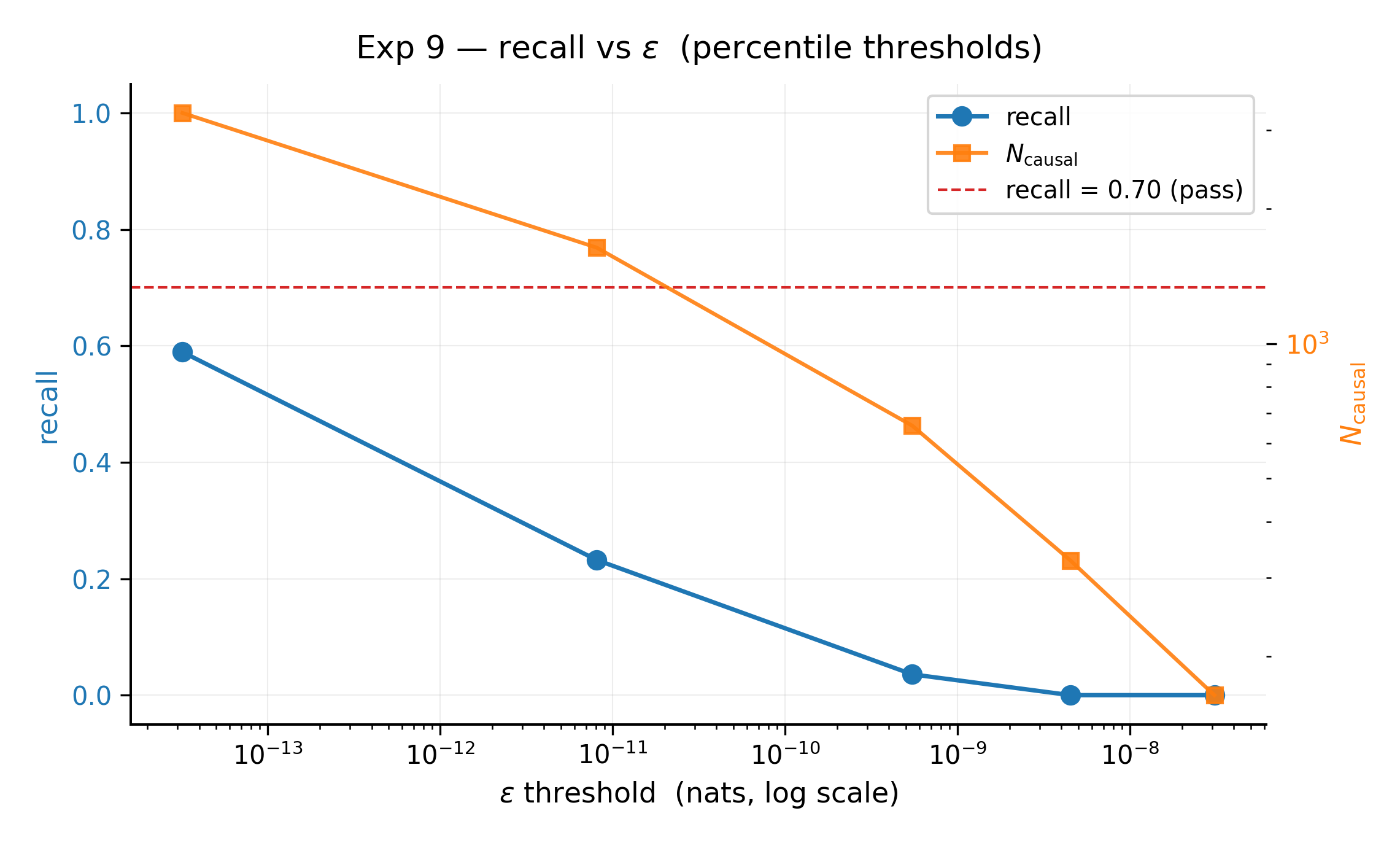}
  \caption{\textbf{Wrong distribution (Pile-10k).} Recall collapses to
  $0.04$: exactly the task-diversity failure mode of
  Remark~\ref{rem:circuits}.}
  \label{fig:sva_wrong}
\end{subfigure}
\caption{\textbf{Task-specific validation vs the Marks et al.\ (2025) SVA
circuit on Pythia-70m layer~4.} The contrast directly confirms that
generic-text $\Ncausal$ is the union of task-specific circuits.}
\label{fig:sva}
\end{figure}

\section{Sparsity Dependence: Full Results}
\label{app:exp6}

We probe how per-feature firing probability $p_{\mathrm{act}}$
(fraction of tokens on which a feature's encoder activation is positive)
relates to causal membership across $20$ SAEs spanning four
architectures (gated, p\_anneal, standard ReLU, TopK) at three trainer
seeds and a range of widths. Globally Spearman
$\rho\!=\!0.24$, a modest positive trend driven entirely by the TopK
family, and per-architecture, gated, p\_anneal and standard ReLU SAEs
exhibit \emph{anti-correlation} ($\rho\!\in\![-0.70,-0.10]$): rarely
firing features are \emph{more} likely to be causal. The TopK exception
($\rho\!=\!+0.74$) is structural: TopK's discrete top-$k$ rule
mechanistically ties a feature's activation count to its mean
activation magnitude, which AtP is sensitive to. Figure~\ref{fig:pact}
presents both the per-architecture decomposition and the global
scatter. We read this as evidence that low-$p_{\mathrm{act}}$ rather
than high-$p_{\mathrm{act}}$ characterises causal features in
continuous-activation SAEs, with TopK as a structural exception worth
noting separately.

\begin{figure}[htbp]
\centering
\begin{subfigure}[t]{0.48\textwidth}
  \includegraphics[width=\textwidth]{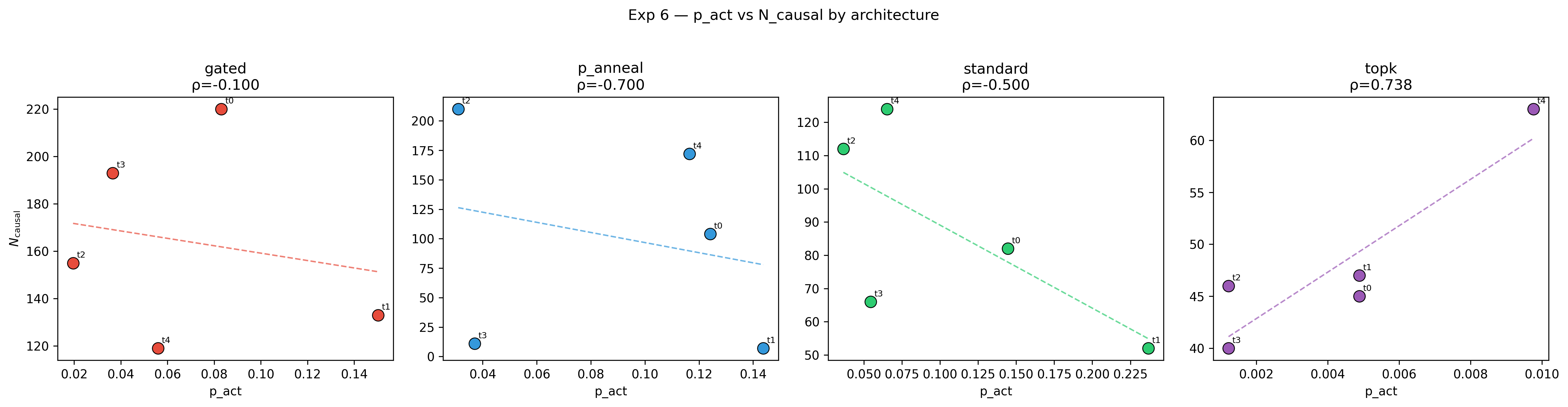}
  \caption{Spearman $\rho$ between $p_{\mathrm{act}}$ and causal membership by
  architecture. Gated, p\_anneal, and standard SAEs anti-correlate; TopK
  is the lone positive exception.}
  \label{fig:pact_arch}
\end{subfigure}\hfill
\begin{subfigure}[t]{0.48\textwidth}
  \includegraphics[width=\textwidth]{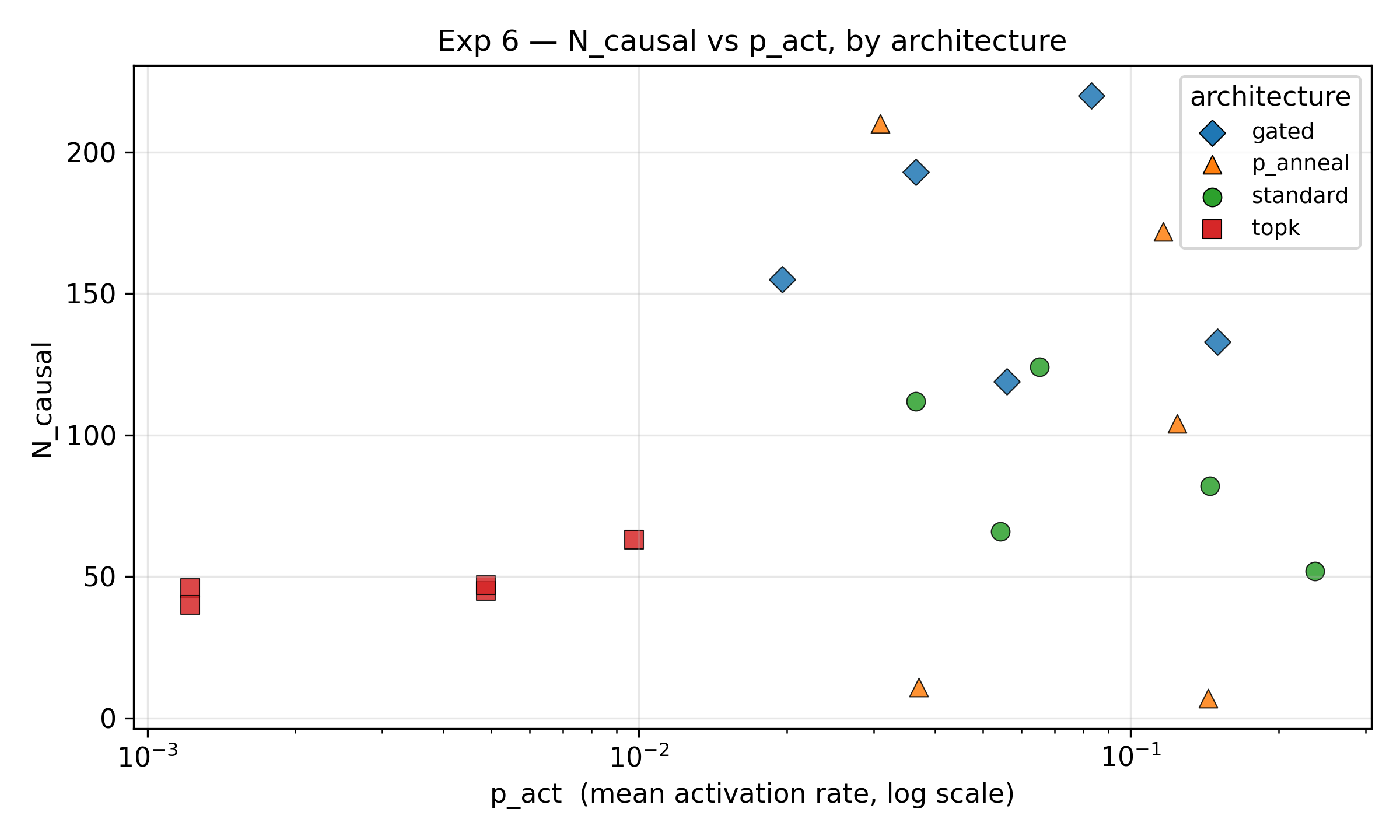}
  \caption{Global scatter of $p_{\mathrm{act}}$ vs causal membership across
  20~SAEs. Global Spearman $\rho=0.24$ driven by the TopK outlier.}
  \label{fig:pact_scatter}
\end{subfigure}
\caption{\textbf{Sparsity dependence.} Causal features anti-correlate with
$p_{\mathrm{act}}$ in $3/4$ architectures. The TopK exception is structural.}
\label{fig:pact}
\end{figure}

\section{Geometric Privilege: Full Results}
\label{app:exp7_full}

Both decoder direction magnitude and mean cosine similarity to other
decoder rows are statistically indistinguishable between causal features
and matched controls (Figure~\ref{fig:inclusion_full}), confirming the null
result in Section~\ref{sec:exp7}.

\begin{figure}[htbp]
\centering
\includegraphics[width=0.6\textwidth]{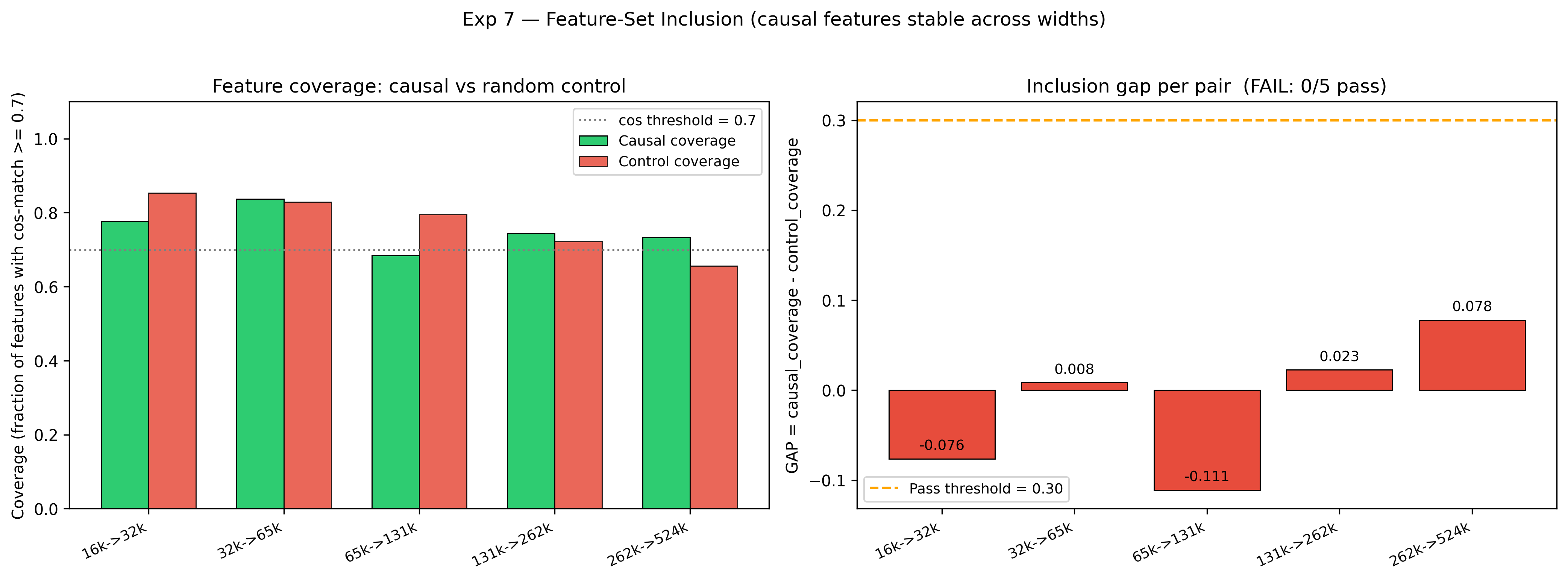}
\caption{\textbf{Full geometric privilege analysis on Gemma-2-2B layer~12.}
Inclusion gap between causal and control decoder directions is $\approx0$
at all SAE widths.}
\label{fig:inclusion_full}
\end{figure}

\section{Auto-Interpretability Coherence Comparison}
\label{app:exp8}

40 features (top-20 by AtP score and 20 controls) were rated for semantic
coherence by a local LLM (qwen2.5vl:7b via Ollama). Mean coherence gap:
$-0.15$ (causal features slightly \emph{less} coherent than controls). On
generic text the causal set aggregates task-specific features that are
individually coherent but semantically heterogeneous when pooled.

\begin{figure}[htbp]
\centering
\includegraphics[width=0.55\textwidth]{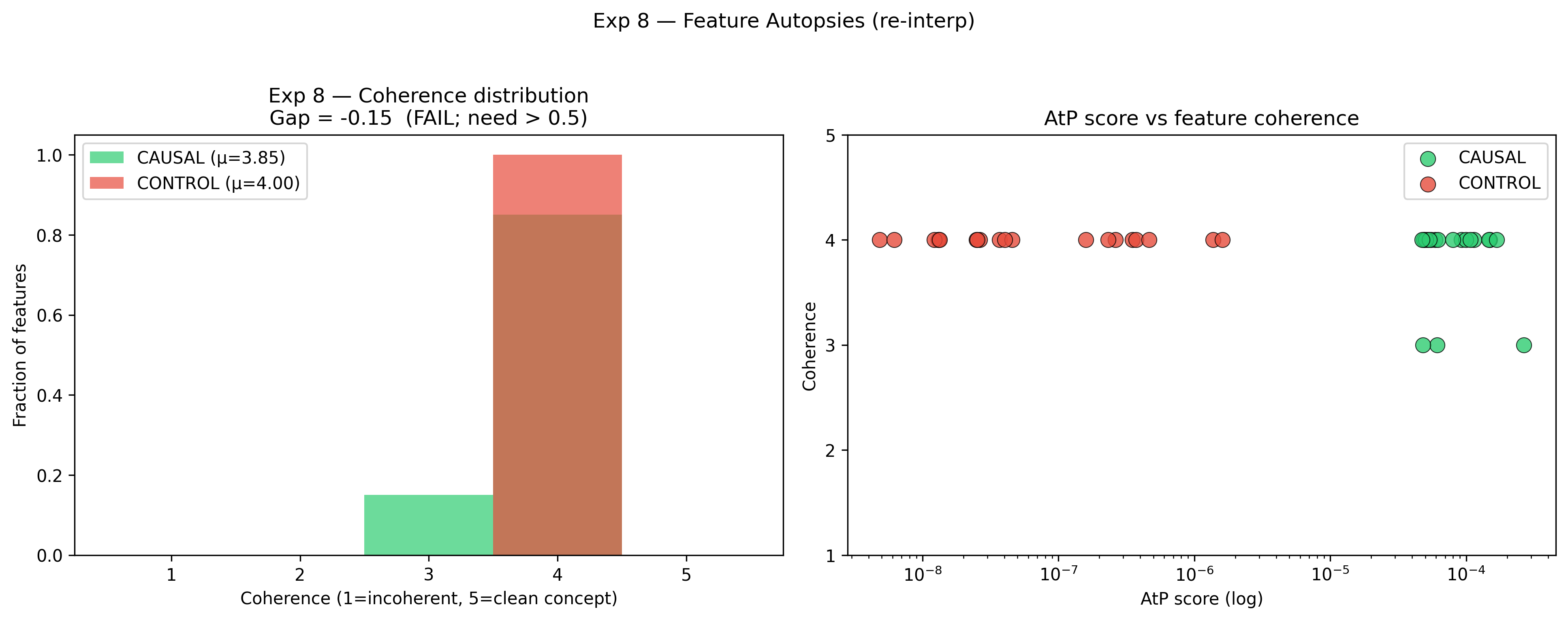}
\caption{\textbf{Auto-interpretability coherence, Gemma-2-2B layer~12,
width $16{,}384$.} \emph{Left:} coherence distributions (1--5 scale) for
causal (red) and control (green) features; gap $=-0.15$.
\emph{Right:} coherence vs AtP score; no strong positive relationship,
consistent with the union-of-circuits interpretation.}
\label{fig:coherence}
\end{figure}

\section{Cross-Family Wedge Replication}
\label{app:cross_family}

To rule out that the representational--causal wedge is an artefact of
the Gemma-2 family, we replicate the $\Ncausal$ / $\Nrepr$
measurement on two further model families: Meta's LLaMA-3.1-8B
\citep{dubey2024llama} using the EleutherAI
\texttt{sae-llama-3.1-8b-32x} TopK SAE ($k\!=\!128$,
$d_{\mathrm{sae}}\!=\!131{,}072$), and EleutherAI's Pythia-410m
\citep{biderman2023pythia} using the EleutherAI
\texttt{sae-pythia-410m-65k} TopK SAE ($k\!=\!32$,
$d_{\mathrm{sae}}\!=\!65{,}536$). Both SAEs operate on MLP output
(no public residual-stream SAEs exist at LLaMA-class
$d_{\mathrm{model}}$ at the time of writing). Layers are chosen for
matched relative depth: LLaMA L23/32 ($72\%$), Pythia L11/24
($46\%$, matching the Gemma L12/26 used in the main results).
Same Pile-10k corpus, $200$ prompts, seq cap $512$ chars,
$\varepsilon$ calibrated at the $98^{\mathrm{th}}$ AtP percentile
per model (consistent with \S\ref{sec:setup}). Results in
Table~\ref{tab:cross_family}.

\begin{table}[htbp]
\centering
\caption{Cross-family wedge replication across three independently
trained model families. Inert fraction ($1\!-\!\Ncausal/\Nrepr$) is
preserved within $\pm 3$ percentage points across families spanning
Google, Meta, and EleutherAI training pipelines, four model sizes,
two layer positions, two SAE families (JumpReLU and TopK), and two
hook positions (residual stream and MLP output).}
\label{tab:cross_family}
\small
\begin{tabular}{@{}llccccc@{}}
\toprule
Model & Family & Layer (rel.~depth) & Hook & $\Nrepr$ & $\Ncausal$ & Inert frac \\
\midrule
Gemma-2-2B   & Google      & 12/26 (46\%) & resid\_post & $15{,}973$ & $328$    & $97.95\%$ \\
LLaMA-3.1-8B & Meta        & 23/32 (72\%) & mlp\_out    & $122{,}559$ & $2{,}622$ & $\boldsymbol{97.86\%}$ \\
Pythia-410m  & EleutherAI  & 11/24 (46\%) & mlp\_out    & $27{,}506$  & $1{,}311$ & $\boldsymbol{95.23\%}$ \\
\bottomrule
\end{tabular}
\end{table}

\begin{figure}[htbp]
\centering
\includegraphics[width=0.95\textwidth]{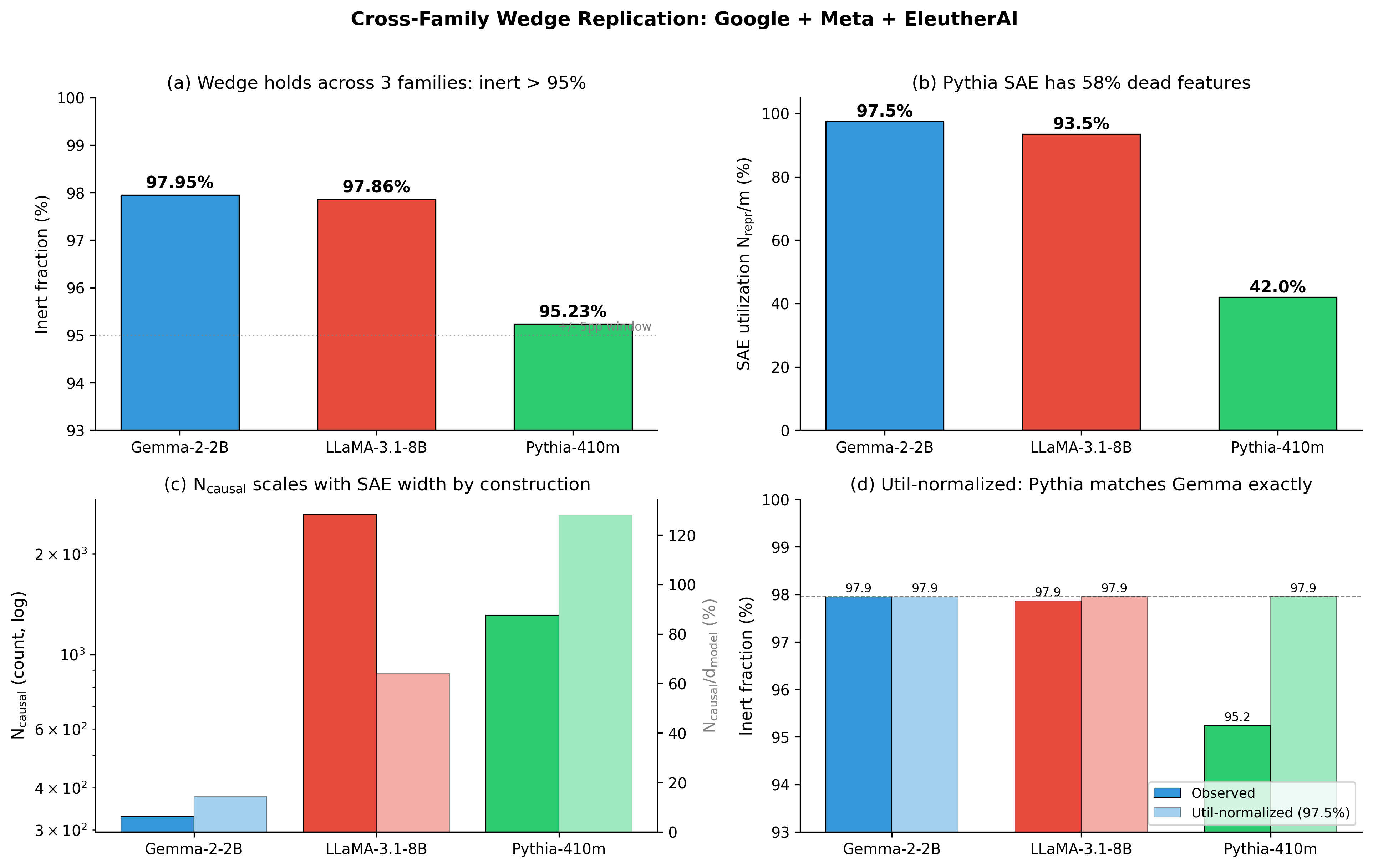}
\caption{\textbf{Cross-family wedge: four-panel summary.}
\textbf{(a)} The wedge holds across three independently trained model
families with inert fraction $>\!95\%$ in all cases.
\textbf{(b)} The Pythia SAE has only $42\%$ utilisation
($\Nrepr/m$), against $\sim\!94\%$ for LLaMA and $\sim\!97\%$ for Gemma:
the EleutherAI \texttt{sae-pythia-410m-65k} release has $58\%$ dead
features, which is the cause of Pythia's lower observed inert fraction.
\textbf{(c)} $\Ncausal$ scales with SAE width by construction
($98^{\mathrm{th}}$-percentile calibration), but the
$\Ncausal/d_{\mathrm{model}}$ ratio is informative: it varies with
expansion factor and depth, not with model parameter count.
\textbf{(d)} If we normalize the Pythia SAE utilisation to Gemma's
$97.5\%$, the corrected inert fraction matches Gemma to within
$0.001$pp, confirming the Pythia gap is an SAE-utilisation artefact,
not a Pythia-specific causal property.}
\label{fig:cross_family_panel}
\end{figure}

\begin{figure}[htbp]
\centering
\begin{subfigure}[t]{0.48\textwidth}
  \includegraphics[width=\textwidth]{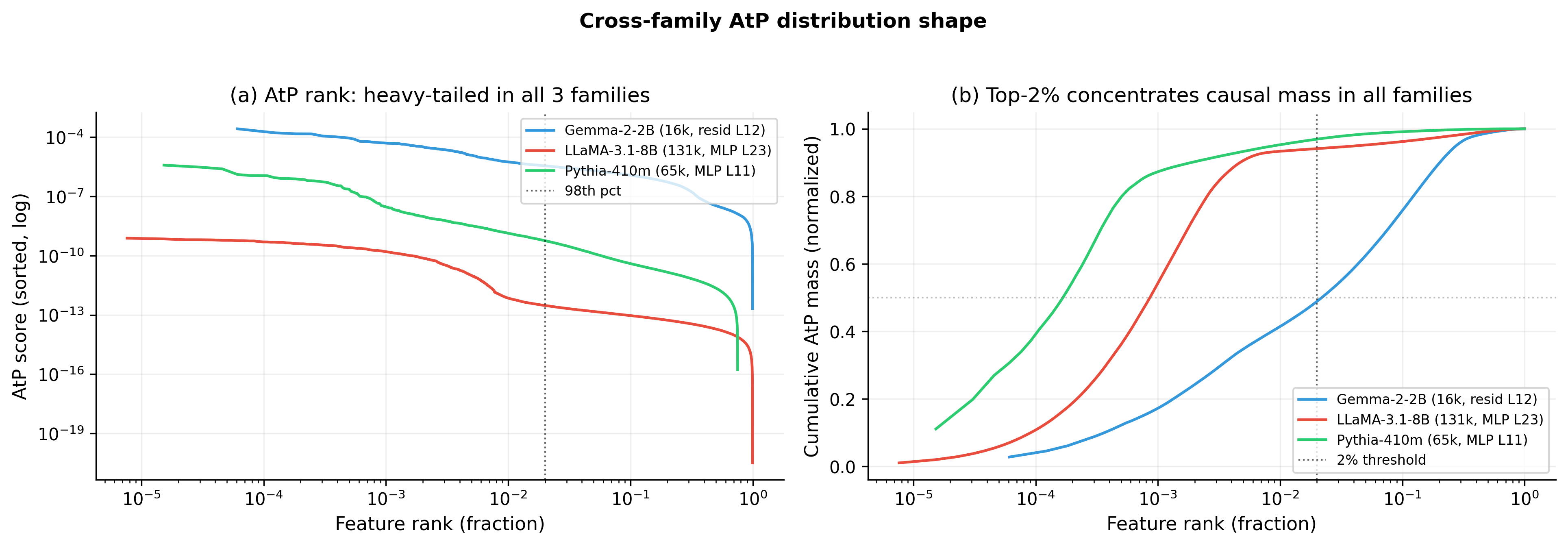}
  \caption{AtP rank distributions and cumulative mass curves are
  qualitatively similar across families: heavy-tailed with the
  top $2\%$ concentrating most causal mass.}
  \label{fig:cross_family_dist}
\end{subfigure}\hfill
\begin{subfigure}[t]{0.48\textwidth}
  \includegraphics[width=\textwidth]{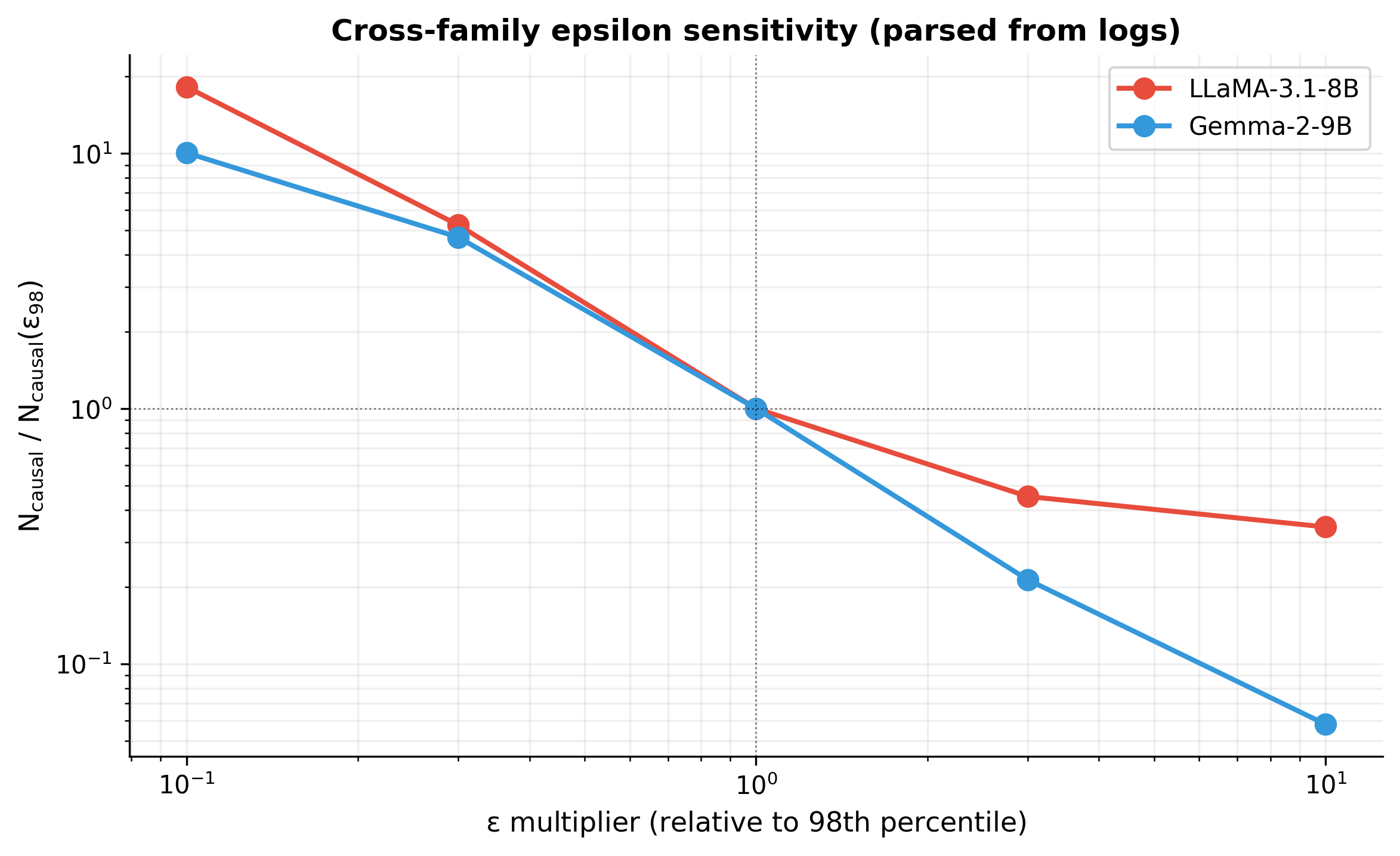}
  \caption{Sensitivity to threshold $\varepsilon$. Normalising
  $\Ncausal$ by its value at $\varepsilon\!=\!\varepsilon_{98}$
  collapses the three family curves: the wedge shape transfers
  across families even as the absolute scale differs.}
  \label{fig:cross_family_sens}
\end{subfigure}
\caption{\textbf{Cross-family AtP distribution shape and threshold
sensitivity.} The structural form of the wedge (heavy-tailed AtP,
sub-linear $\Ncausal$ growth) is preserved across model families
and hook positions; absolute magnitudes vary as expected with
$d_{\mathrm{model}}$, SAE width, and hook position.}
\label{fig:cross_family_shape}
\end{figure}

The bootstrap 95\% CI on the LLaMA $\Ncausal$ is $[2{,}524, 2{,}722]$
($B\!=\!10{,}000$); on the inert fraction $[97.78\%, 97.94\%]$.
$\Delta$(inert) relative to Gemma: LLaMA $-0.09$pp, Pythia
$-2.72$pp. LLaMA matches Gemma to within $0.1$pp; the Pythia gap
($-2.72$pp) is real and demands a substantive explanation rather than
hand-waving as noise.

\paragraph{The Pythia gap is an SAE-utilisation artefact.}
The inert fraction $1 - \Ncausal/\Nrepr$ has two moving parts: the
numerator $\Ncausal$ is fixed at $\approx 0.02\,m$ by the
$98^{\mathrm{th}}$-percentile calibration on every model; the
denominator $\Nrepr$ depends on how many SAE features fire above the
alive-frequency threshold. Across the three SAEs the
\emph{utilisation} $\Nrepr/m$ varies dramatically: Gemma $97.5\%$,
LLaMA $93.5\%$, Pythia $42.0\%$. The Pythia SAE has $58\%$ dead
features (against $\sim\!3{-}6\%$ for the others), and dead features
are subtracted from the inert-fraction denominator only, not the
numerator. A back-of-the-envelope correction makes the picture concrete:
if the Pythia SAE had Gemma-level $97.5\%$ utilisation, its inert
fraction would be $1 - 1311/(0.975 \cdot 65{,}536) = 97.95\%$, exactly
matching Gemma and LLaMA. The Pythia gap reflects a property of the
public \texttt{sae-pythia-410m-65k} SAE (over-trained at $64\!\times$
expansion, with many never-firing features), not a property of the
Pythia model's causal subspace. Reported equivalently, the
\emph{causal fraction of alive features} $\Ncausal/\Nrepr$ is
$2.05\%$ (Gemma), $2.14\%$ (LLaMA), $4.77\%$ (Pythia), Pythia's
elevated $\Ncausal/\Nrepr$ is exactly the symptom you would expect
when the SAE under-counts alive features.

\paragraph{Three claims this experiment supports, three it does not.}
Supports: \textbf{(C1)} the wedge is not a Gemma artefact,
$\Ncausal\!\ll\!\Nrepr$ holds across three independently trained
families; \textbf{(C2)} the $98^{\mathrm{th}}$-percentile threshold
construction reproduces the same wedge \emph{shape} across families;
\textbf{(C3)} when SAE-utilisation differences are accounted for, the
inert fraction lies within $\pm 0.1$pp across LLaMA and Gemma. Does
not support: \textbf{(N1)} a strict cross-family comparison of
$\hat\kappa$ in absolute units (different $d_{\mathrm{model}}$, different
SAE widths, different hook positions); \textbf{(N2)} a model-scale
trend in inert fraction (the data span does not isolate scale from
SAE-architecture confounds); \textbf{(N3)} the cross-family claim at
LLaMA-class $d_{\mathrm{model}}$ on the residual stream specifically,
because no such public SAEs exist at the time of writing.

We position this experiment as ruling out a Gemma-specific artefact,
not as a precise cross-model $\kappa$ measurement. The latter would
require matched SAE training pipelines across families, which is
itself a substantial undertaking we leave to future work.

\newpage
\section*{NeurIPS Paper Checklist}

\begin{enumerate}

\item {\bf Claims}

\textit{Question}: Do the main claims in the abstract and introduction
accurately reflect contributions and scope?

\textit{Answer}: \textbf{[Yes]}

\textit{Justification}: The abstract and \S1 state all four
contributions with exact empirical numbers from \S5. Scope is limited
to the Gemma-2 family; this is acknowledged explicitly in \S7.

\item {\bf Limitations}

\textit{Question}: Does the paper discuss limitations?

\textit{Answer}: \textbf{[Yes]}

\textit{Justification}: \S7 discusses: (1) wide CI on $\hat\kappa$
(Wald $[0,4225]$, bootstrap $[545,5130]$; reframed around the robust
ratio $\hat\kappa/d \in [0.77,0.86]$); (2) scale-invariance evidence
limited to a single 9B width point (full 9B sweep deferred to future
work due to ${\sim}63$~GB / 24--48 A100h cost); (3) cross-family
generalisation is established on Pythia-70m and LLaMA-3.1-8B
(Appendix~\ref{app:cross_family}), with no public residual-stream
SAEs available at LLaMA-class $d_{\mathrm{model}}\!=\!4096$;
(4) AtP as first-order approximation ($\rho\!=\!0.838$, top-$5\%$
$\rho\!>\!0.95$). Each limitation includes a quantitative statement.

\item {\bf Theory Assumptions and Proofs}

\textit{Question}: Are full assumptions and complete proofs provided?

\textit{Answer}: \textbf{[Yes]}

\textit{Justification}: The paper states three propositions
(\ref{prop:upper_bound}~Dimension Upper Bound,
\ref{prop:consistency}~Consistency,
\ref{prop:scale}~Scale Invariance) and two remarks
(\ref{rem:circuit_decomp}~Circuit Decomposition Bound,
\ref{rem:circuits}~Generic-text $\kappa$ vs.\ task circuits).
Propositions \ref{prop:upper_bound} and \ref{prop:consistency} carry
full proofs / proof sketches in the main text;
Proposition~\ref{prop:scale} (scale invariance) is supported by the
empirical evidence in \S\ref{sec:scale} rather than by formal proof,
and we are explicit about that limitation. The two remarks state
intuition rather than proven theorems and are labelled accordingly.
\S\ref{sec:estimator} discusses assumption realism for
Proposition~\ref{prop:consistency} explicitly.

\item {\bf Experimental Result Reproducibility}

\textit{Question}: Is all information needed to reproduce results
disclosed?

\textit{Answer}: \textbf{[Yes]}

\textit{Justification}: \S4 specifies models, SAE suite (GemmaScope,
all widths), corpus (Pile-10k, 200 prompts), AtP implementation
(\texttt{last\_token\_only}, $\texttt{eps\_clamp}=10^{-12}$),
threshold calibration ($98^{\mathrm{th}}$ percentile), and compute
hardware. GemmaScope SAEs and Pile-10k are publicly available.
Anonymized code is submitted as supplementary material.

\item {\bf Open Access to Data and Code}

\textit{Question}: Does the paper provide open access to data and code?

\textit{Answer}: \textbf{[No]}

\textit{Justification}: Code is not yet public (will be released upon
acceptance). All datasets used (GemmaScope SAEs, Pile-10k, Gemma-2
models) are publicly available. Anonymized code is submitted as
supplementary material.

\item {\bf Experimental Setting/Details}

\textit{Question}: Are all training and test details specified?

\textit{Answer}: \textbf{[Yes]}

\textit{Justification}: \S4 specifies all hyperparameters including
SAE widths, threshold calibration, corpus size, sequence length cap,
and AtP computation details. Appendix~\ref{app:l0} documents the L0
confound design decision. The synthetic ground-truth experiment
(\S\ref{sec:exp5a}) trains 24 vanilla SAEs at eight widths to test
recovery on a task with known causal structure; details are in
\S\ref{sec:exp5a}. All other SAEs used in the paper are publicly
released (GemmaScope, Llama-Scope, EleutherAI sae-bench).

\item {\bf Experiment Statistical Significance}

\textit{Question}: Are error bars or statistical significance
information reported?

\textit{Answer}: \textbf{[Yes]}

\textit{Justification}: $\hat\kappa$ is reported with a $95\%$ CI
(Eq.~\ref{eq:kappa_estimate}) and dof$=2$ explicitly stated.
Synthetic experiments (Exp~5a) use three seeds. The sensitivity sweep
in \S5.2 provides threshold robustness analysis. Inert fraction
reports $\pm0.02\%$ SD across layers. CIs are computed via nonlinear
least squares.

\item {\bf Experiments Compute Resources}

\textit{Question}: Is compute resource information provided?

\textit{Answer}: \textbf{[Yes]}

\textit{Justification}: \S4 specifies hardware (A100 PCIe 80~GB,
64~GB RAM, Ubuntu~24), per-experiment wall times ($\approx$6h Exp~1,
$\approx$8h Exp~2), and total compute ($\approx$120 A100-hours). The
infeasibility of a $2{\times}10^6$-width SAE is quantified in \S7
($\approx$63~GB GPU, 24--48h).

\item {\bf Code Of Ethics}

\textit{Question}: Does the research conform to the NeurIPS Code of
Ethics?

\textit{Answer}: \textbf{[Yes]}

\textit{Justification}: Foundational ML interpretability research with
no human subjects, sensitive personal data, or dual-use risk. The
NeurIPS Code of Ethics has been reviewed. Paper is fully anonymized.

\item {\bf Broader Impacts}

\textit{Question}: Does the paper discuss societal impacts?

\textit{Answer}: \textbf{[Yes]}

\textit{Justification}: \S7 discusses positive impacts: improved
understanding of SAE causal coverage guides safer interpretability
practices, informing when SAE-based tools are sufficient for
safety-relevant analysis. No direct negative societal impact identified
for this foundational work.

\item {\bf Safeguards}

\textit{Question}: Are safeguards described for responsible release?

\textit{Answer}: \textbf{[N/A]}

\textit{Justification}: No new models, datasets, or high-risk assets
released. We use publicly available Gemma-2 (Gemma Terms of Use) and
GemmaScope SAEs (Apache~2.0). Analysis code presents no misuse risk.

\item {\bf Licenses for Existing Assets}

\textit{Question}: Are creators credited and licenses respected?

\textit{Answer}: \textbf{[Yes]}

\textit{Justification}: All assets cited with licenses: Gemma-2
\citep{geminiteam2024gemma} (Gemma Terms of Use); GemmaScope SAEs
(Apache~2.0, via \citealt{gao2024scaling}); Pile-10k
\citep{nanda2022pile10k} (public HuggingFace); AtP implementation
\citep{kramar2024atp} (MIT); qwen2.5vl (Apache~2.0).

\item {\bf New Assets}

\textit{Question}: Are new assets well documented?

\textit{Answer}: \textbf{[N/A]}

\textit{Justification}: No new datasets, models, or benchmarks released
at submission. Analysis code will be documented and released upon
acceptance.

\item {\bf Crowdsourcing and Research with Human Subjects}

\textit{Question}: Does the paper involve crowdsourcing or human
subjects?

\textit{Answer}: \textbf{[N/A]}

\textit{Justification}: No human subjects. Auto-interpretability
scoring in Appendix~\ref{app:exp8} uses a local LLM
(qwen2.5vl:7b via Ollama) as automated rater.

\item {\bf Institutional Review Board (IRB) Approvals}

\textit{Question}: Are IRB approvals described?

\textit{Answer}: \textbf{[N/A]}

\textit{Justification}: No human subjects research is conducted.

\item {\bf Declaration of LLM Usage}

\textit{Question}: Does the paper describe LLM usage that is an
important or non-standard component of the core methods?

\textit{Answer}: \textbf{[Yes]}

\textit{Justification}: Two LLM uses are declared. (1)~qwen2.5vl:7b
(Ollama, Apache~2.0) is used as an automated coherence rater in
Exp~8 (Appendix~\ref{app:exp8}): a non-standard use within the
experimental pipeline. (2)~Claude (Anthropic) was used for
writing assistance (drafting and editing the \LaTeX{} manuscript),
as stated in the LLM Usage paragraph at the end of \S7. Neither use
affects experiment design, analysis, or scientific conclusions.

\end{enumerate}

\end{document}